\documentclass[10pt, conference, letterpaper]{IEEEtran}

\usepackage{latexsym,bm}
\usepackage{xcolor}

\usepackage{makecell}
\usepackage[pdftex]{graphicx}
\usepackage{graphicx}
\usepackage{stfloats}
\usepackage{epsfig}
\usepackage{booktabs}
\usepackage{multirow}
\usepackage[ruled,linesnumbered,noend]{algorithm2e}
\usepackage{amsmath}
\usepackage{amssymb}
\usepackage{amsthm}
\usepackage{subfigure}
\usepackage{graphicx}
\usepackage{array}
\usepackage{comment} 
\usepackage{cite}

\newcommand{\Identity}{{\rm I\kern-.2em l}}
\newcommand{\Expect}{\mathbb{E}}

\IEEEoverridecommandlockouts 
\ifCLASSINFOpdf
 
\else

\fi

\begin{document}
\title{Cost-Effective Federated Learning Design
}%

\author{\IEEEauthorblockN{Bing Luo\IEEEauthorrefmark{1}\IEEEauthorrefmark{4},
Xiang Li\IEEEauthorrefmark{2}, Shiqiang Wang\IEEEauthorrefmark{3}, Jianwei Huang\IEEEauthorrefmark{2}\IEEEauthorrefmark{1}, 
Leandros Tassiulas\IEEEauthorrefmark{4}}
\IEEEauthorblockA{\IEEEauthorrefmark{1}Shenzhen Institute of Artificial Intelligence and Robotics for Society, China\\
\IEEEauthorrefmark{2}School of Science and Engineering, The Chinese University of Hong Kong, Shenzhen, China\\
\IEEEauthorrefmark{3}IBM T. J. Watson Research Center, Yorktown Heights, NY, USA\\
\IEEEauthorrefmark{4}Department of Electrical Engineering and Institute for Network Science, Yale University, USA\\
Email: \{\IEEEauthorrefmark{1}\IEEEauthorrefmark{4}luobing, \IEEEauthorrefmark{2}lixiang, \IEEEauthorrefmark{2}\IEEEauthorrefmark{1}jianweihuang\}@cuhk.edu.cn,
\IEEEauthorrefmark{3}wangshiq@us.ibm.com,
\IEEEauthorrefmark{4}leandros.tassiulas@yale.edu}\vspace{-2mm}
\thanks{The research of B. Luo, X. Li, and J. Huang was supported by the Shenzhen Institute of Artificial Intelligence and Robotics for Society (AIRS), the Presidential Fund from the Chinese University of Hong Kong, Shenzhen, and the AIRS International Joint Postdoctoral Fellowship. The research of L. Tassiulas was partially supported by projects ARO W911NF1810378 and ONR  N00014-19-1-2566. (Corresponding author: Jianwei Huang.)}
}

\maketitle

\begin{abstract}
Federated learning (FL) is a distributed learning paradigm that enables a large number of devices to collaboratively learn a model %
without sharing their raw data. %
Despite its practical efficiency and effectiveness, %
the %
iterative on-device learning process  
incurs a considerable cost in terms of learning time and energy consumption, 
which depends crucially on the number of selected clients and the number of local iterations in each training round. 
In this paper, %
we analyze how to 
design adaptive FL that optimally chooses
these essential control variables  %
to minimize the %
total cost  %
while ensuring convergence. 
Theoretically, we  analytically  establish  the relationship  between  the  %
total   cost  and  the control variables with the convergence upper bound. To efficiently solve the cost minimization problem,  %
we develop a low-cost sampling-based algorithm to learn the   convergence related unknown parameters. %
We derive important solution properties that effectively identify the %
design principles %
for different metric preferences.  
Practically, we  evaluate our theoretical results  both in a simulated environment and on a hardware prototype. Experimental evidence verifies our derived properties and demonstrates that our proposed solution %
achieves near-optimal performance for various datasets, different machine learning models, %
and heterogeneous system settings. %

\end{abstract}

\IEEEpeerreviewmaketitle

\section{Introduction}

\emph{Federated learning (FL)} has recently emerged as an attractive distributed learning paradigm, which
enables many clients\footnote{Depending on the type of clients, FL can be categorized into cross-device FL and cross-silo FL (clients are companies or organizations, etc.) \cite{kairouz2019advances}. This paper focuses on the former and we use ``device'' and ``client'' interchangeably.}
to collaboratively train a  model under the coordination of a central server, while keeping the training data decentralized and private\cite{mcmahan2017communication,konevcny2016federated1,kairouz2019advances,yang2019federated,li2020federated,park2019wireless}. %
In FL settings, %
the training data are generally massively distributed over a large number of devices, and the communication between the server and clients are typically operated at lower rates compared to datacenter settings. 
These unique features necessitate  FL algorithms that 
perform \emph{multiple local iterations} in parallel on \emph{a fraction of randomly sampled clients} and then aggregate the resulting model update via the central server periodically \cite{mcmahan2017communication}. 
FL has demonstrated empirical success and theoretical convergence guarantees in various heterogeneous settings, e.g., unbalanced and non-i.i.d. data distribution \cite{sattler2019robust,bonawitz2019towards,mcmahan2017communication, smith2017federated,
li2018federated, li2019convergence}. %

Because model training and information transmission for on-device FL can be both time and energy consuming,
it is necessary and important to analyze the \emph{cost} that is incurred for completing a given FL task. 
In  general,  the  cost  of  FL  includes multiple components such as  learning  time  %
and  energy consumption \cite{tran2019federated}. %
The importance of different cost components depends on the characteristics of FL systems and applications. 
For example, in a solar-based sensor network, energy consumption is the major concern for the sensors to participate in FL tasks, whereas in a multi-agent search-and-rescue task where the goal is to collaboratively learn an unknown map, achieving timely result would be the first priority. 
Therefore, a cost-effective FL design needs to \emph{jointly optimize various cost components} (e.g., learning  time and  energy consumption) \emph{for different preferences}.

A way of optimizing the cost is to adapt control variables in the FL process to achieve a properly defined objective. For example, some existing works have considered the adaptation of communication interval (i.e., the number of local iterations between two %
global aggregation rounds) for communication-efficient FL with convergence guarantees \cite{wang2019adaptive,wang2018adaptive}. 
However, a limitation in these works is that they only adapt a single control variable (i.e., communication interval) in the FL process and ignore other essential aspects, such as the number of participating clients in each round, which can have a significant impact on the energy consumption.

In this paper, we consider a \emph{multivariate} control problem for cost-efficient FL with convergence guarantees. To minimize the expected cost, we develop an algorithm that adapts various control variables in the FL process to achieve our goal. Compared to the univariate setting in existing works, our problem is much more challenging due to the following reasons: \emph{1)}~The choices of control variables are tightly coupled. \emph{2)}~The relationship between the control variables and the learning convergence rate has only been captured by an upper bound with unknown coefficients in the literature. \emph{3)} Our cost objective includes multiple components (e.g., time and energy) which can have different importance depending on the system and application scenario, whereas existing works often consider a single optimization objective such as minimizing the communication overhead.

\begin{figure}[!t]
	\centering
	\includegraphics[width=8cm,height=5cm]{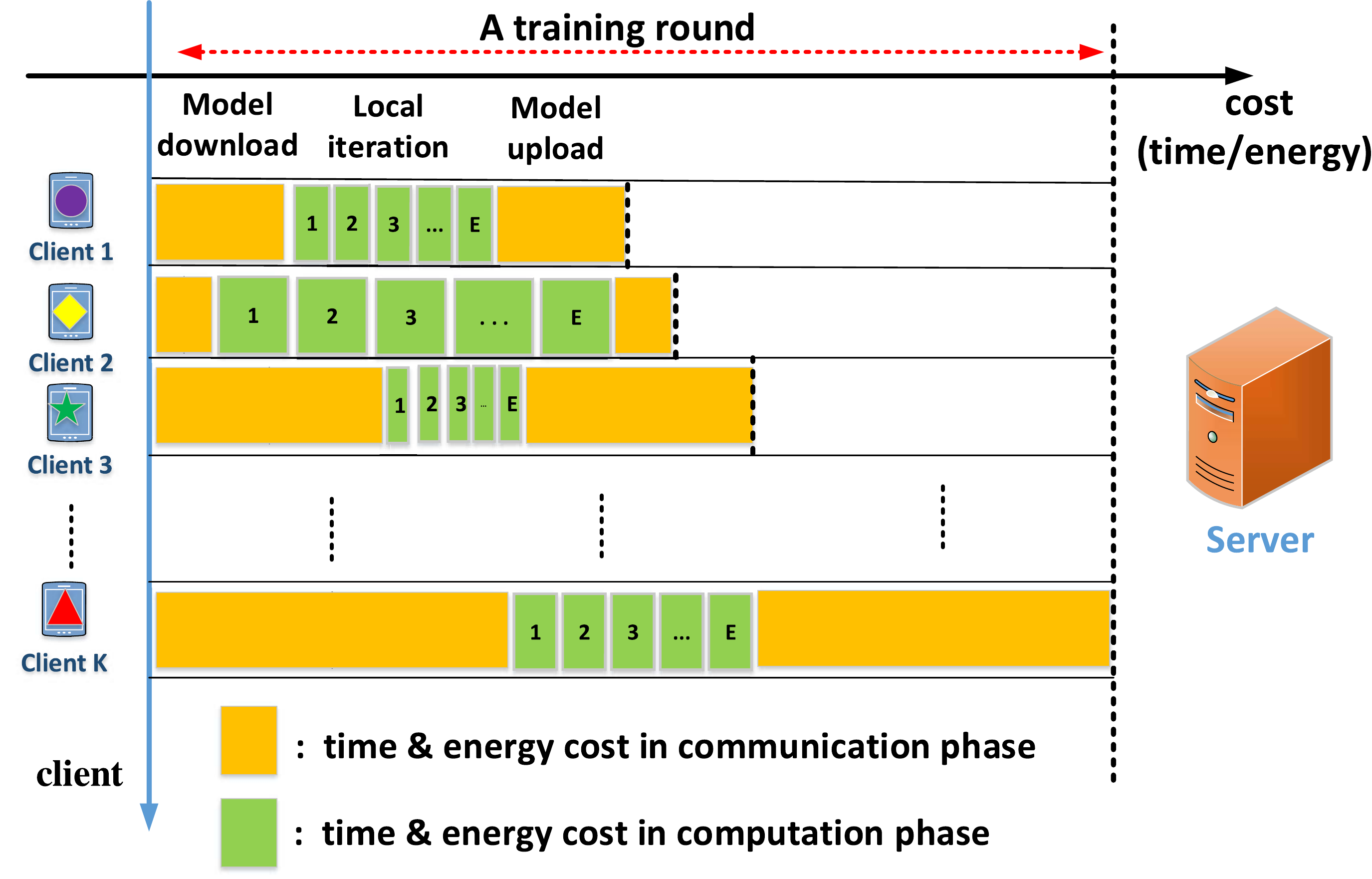}
	\caption{A typical federated learning round with $K$ sampled clients and $E$ steps of local iterations.}%
	\label{fig:intro}
\end{figure}

As illustrated in Fig.~1, %
we consider the number of participating clients (${K}$) and the number of local iterations (${E}$) in each FL round as our control variables. A  similar methodology can be applied to analyze problems with other control variables as well. We analyze, for the first time, how to design adaptive FL that optimally chooses $K$ and $E$ to minimize the %
total cost %
while ensuring convergence.  %
Our main contributions are as follows:

\begin{itemize}
    \item  \emph{Optimization Algorithm:}  {We establish the analytical relationship between the %
    total cost,} %
control variables, and convergence upper bound for strongly convex objective functions, based on which an optimization problem for total cost minimization is formulated and analyzed. We propose a sampling-based algorithm to learn the unknown parameters in the convergence bound with marginal estimation overhead.  %
We show that our optimization problem is 
\emph{biconvex} with respect to $K$ and $E$, {and develop efficient ways to solve it based on closed-form expressions. } %
    \item \emph{Theoretical Properties%
     :} We theoretically obtain important properties 
      that effectively identify the design principles for  %
     different %
     optimization goals. Notably, %
     the choice of $K$ leads to an interesting trade-off between learning time reduction and energy saving, with a large $K$ favoring the former while a small $K$ benefiting the later. %
     Nevertheless, we show that a relatively
     low device participation rate does not severely slow down the learning. %
{For the choice of $E$, we show that neither  a too small or too large $E$ is good for cost-effectiveness. The optimal value of $E$ also depends on the relationship between computation and communication costs.} %
     
    \item \emph{Simulation and Experimentation:} We evaluate our theoretical results %
    with real datasets,  
both in a simulated environment and on a hardware prototype with $20$ Raspberry Pi devices. %
Experimental results verify our design principles and derived properties of $K$ and $E$. 
They also demonstrate that our proposed optimization algorithm {provides near-optimal solution  for both real and synthetic datasets with non-i.i.d. data distributions. Particularly, we highlight that our approach works well with both convex non-convex machine learning models empirically.  %
}
\end{itemize}

\section{Related Work}
FL was first proposed in \cite{mcmahan2017communication}, which demonstrated FL's effectiveness of collaboratively learning a model without collecting users' data.
Compared to distributed learning in data centers, FL needs to address several unique challenges, including  non-i.i.d. and unbalanced data, limited communication bandwidth, and limited device availability (partial participation) \cite{kairouz2019advances,li2020federated}. 
 It was suggested that  FL algorithms should operate  synchronously due to its composability with other
techniques such as secure aggregation protocols \cite{bonawitz2016practical}, differential
privacy \cite{avent2017blender}, and model compression \cite{konevcny2016federated}. Hence, we consider synchronous FL in this paper with all the aforementioned characteristics.

The de facto FL algorithm %
is federated averaging (FedAvg), which performs multiple local iterations in parallel %
 on  a subset of 
 devices in each round.
A system-level FL framework was presented in \cite{bonawitz2019towards}, which demonstrates the empirical success of  FedAvg in mobile devices using TensorFlow \cite{abadi2016tensorflow}.  
Recently, a convergence bound of FedAvg was established in \cite{li2019convergence}. 
Other related distributed optimization algorithms are mostly for i.i.d. datasets \cite{yu2018parallel,stich2018local,wang2018cooperative} and full client participation \cite{khaled2019first,smith2017federated,zhou2017convergence}, which do not capture the essence of on-device FL.  
Some extensions of FedAvg considered aspects such as adding a proximal term~\cite{li2018federated} and using accelerated gradient descent methods~\cite{liu2020accelerating}.
These works did not consider optimization for cost/resource efficiency.

Literature in FL cost optimization mainly focused on  learning time and on-device energy consumption. %
The optimization of learning time was studied in \cite{samarakoon2018federated, zhu2018low, nishio2019client,wang2020optimizing,shi2020device,chen2020convergence,han2020adaptive,jiang2020pruning}, and
joint optimization for learning time and energy consumption was considered in \cite{mo2020energy,zeng2020energy,yang2019energy}.  These works considered cost-aware client scheduling \cite{samarakoon2018federated, zhu2018low, nishio2019client,wang2020optimizing}, task offloading\cite{tu2020network} and resource (e.g., transmission power, communication bandwidth, and CPU frequency) allocation \cite{shi2020device, chen2020convergence,tran2019federated,mo2020energy,zeng2020energy,yang2019energy} for \emph{pre-specified} (i.e., non-optimized) design parameters ($K$ and $E$ in our case) of the FL algorithm.  %

The optimization of a single design parameter $E$ or the amount of information exchange, in general, was studied in~\cite{wang2019adaptive,wang2018adaptive,tran2019federated,han2020adaptive,luping2019cmfl,hsieh2017gaia}, most of which assume full client participation and can be infeasible for large-scale on-device FL. A very recent work in \cite{jin2020resource} considered the optimization of both $E$ and client selection for additive per-client costs. {However, the cost related to learning time in our problem is non-additive on a per-client basis, because different clients perform local model updates in parallel. In addition, the convergence bound used in~\cite{jin2020resource} (and also~\cite{tran2019federated}) is for a primal-dual optimization algorithm, which is different from the commonly used FedAvg algorithm  and does not reflect the impact of key FL characteristics such as partial client participation.
The challenge in optimizing both $K$ and $E$ for cost minimization of FedAvg that takes into account all the aforementioned FL characteristics, which also distinguishes our work from the above, is the need %
\emph{to analytically connect the total cost with multiple control variables as well as with the convergence rate}.}

In addition, most existing work on FL are based on simulations, whereas we implement our algorithm in an actual hardware prototype with resource-constrained devices.

\textbf{Roadmap:} We present the system model and problem formulation in Section~\ref{sec:systemModel}. In Section~\ref{sec:optimizationProblem}, we analyze the cost minimization problem and present an algorithm to solve it. We provide theoretical analysis on the solution properties in Section~\ref{sec:property}.
Experimentation results are given in Section~\ref{sec:experimentation} and the
conclusion is presented in Section~\ref{sec:conclusion}.

\section{%
System Model}
\label{sec:systemModel}

We start by  summarizing the basics of FL and its de facto algorithm FedAvg. Then, we present the cost model for a given FL task, and introduce our optimization problem formulation. 

\subsection{Federated Learning}%
Consider a scenario with a large population of mobile
clients that have data for training a machine learning  model. %
Due to data privacy and bandwidth limitation concerns, it is not desirable for clients to disclose and send their raw %
data to a high-performance data center. FL is a decentralized learning framework that aims to resolve this problem. %
Mathematically, FL is the following distributed optimization problem:
\begin{equation}
\label{gl_ob}
\min_{\mathbf{w}}  F\left( \mathbf{w} \right) :=\sum\nolimits_{k = 1}^N{p_k}{F_k}\left( \mathbf{w} \right)
\end{equation}
where the objective $F\left( \mathbf{w} \right)$ is also known as the global loss function, $\mathbf{w}$ is the model parameter vector, $N$ is the total number of devices, and $p_k$ is the weight of the $k$-th device such that $\sum\nolimits_{k = 1}^N p_k=1$. 
Suppose the $k$-th device has $n_k$ training data samples ($\mathbf{x}_{k, 1}, \cdots, \mathbf{x}_{k, n_{k}}$), and the total number of training data samples across $N$ devices is $n :=\sum\nolimits_{k \!=\! 1}^N n_k$, then we have $p_k=\frac{n_k}{n}$.
The local loss function of client $k$ is
\begin{equation}
\label{lo_ob}
{F_k}\left( \mathbf{w} \right) := \frac{1}{{{n_k}}}\sum\limits_{j =1}^{n_k} {{f}\left( \mathbf{w}; \mathbf{x}_{k,j} \right)},
\end{equation}
where $f(\cdot)$ represents a per-sample loss function, e.g., mean square error and cross entropy applied to the output of a model with parameter $\mathbf{w}$ and input data sample $\mathbf{x}_{k,j}$ \cite{wang2019adaptive}. %

FedAvg (Algorithm 1) was proposed in \cite{mcmahan2017communication} to solve \eqref{gl_ob}. In each \emph{round}~$r$, a subset of randomly selected clients $\mathcal{K}^{(r)}$ run $E$ steps\footnote{$E$ is originally defined as epochs of SGD in \cite{mcmahan2017communication}. In this paper we denote $E$ as the number of local iteration for theoretically analysis.} of stochastic gradient decent (SGD) on~\eqref{lo_ob} in parallel, {where $\mathcal{K}^{(r)} \subseteq \{1,2,...,N\}$}. Then, the updated model parameters of these $\left|\mathcal{K}^{(r)}\right|$ clients are sent to and aggregated by the server. This process repeats for many rounds until the global loss converges. {Let $R$ be the total number of rounds, then the total number of iterations for each device is $ER$.} %

While FL has demonstrated its effectiveness in many application scenarios, %
practitioners also need to take into account the \emph{cost} that is incurred for completing a given task.

\begin{algorithm}[t]
\small
	\caption{Federated  Learning Algorithm}
	\label{alg:fedavg}
	\KwIn{$K$, $E$, precision $\epsilon$, initial model $\mathbf{w_0}$}
	\KwOut{Final model parameter $\mathbf{w}_R$}
	\For{$r=0,1,2,..., R$%
	}{	Server randomly selects a subset of clients  $\mathcal{K}^{(r)}$ and sends the current global model parameter $\mathbf{w}_r$ to the selected clients\label{alg:fedavgStep1}\tcp*{Communication}
		
		Each selected client $k \in \mathcal{K}^{(r)} $ in parallel updates $\mathbf{w}_r$ by running $E$ steps  of SGD  on  \eqref{lo_ob} to compute a new model $\mathbf{w}_r^{(k)}$\label{alg:fedavgStep2}\tcp*{Computation}
		
		Each selected client $k \in \mathcal{K}^{(r)}$   sends back the updated model $\mathbf{w}_r^{(k)}$ to the server\label{alg:fedavgStep3}\tcp*{Communication}
		
		Server computes the new global model parameter $\mathbf{w}_{r+1} \leftarrow \frac{\sum_{k \in \mathcal{K}^{(r)}} p_k \mathbf{w}_r^{(k)}}{\sum_{k \in \mathcal{K}^{(r)}} p_k}  $\label{alg:fedavgStep4}\tcp*{Aggregation}
		
		$r \leftarrow r+1$;
	}
\end{algorithm}

\subsection{Cost Analysis of Federated Learning}
The total cost of FL, according to Algorithm~\ref{alg:fedavg}, involves  \emph{learning time}  and  \emph{energy consumption}, both of which are consumed during local computation (Line~\ref{alg:fedavgStep2})  and global communication (Lines~\ref{alg:fedavgStep1} and \ref{alg:fedavgStep3}) in each round.  Before presenting each cost model, we first give the system assumptions. 

\emph{System assumptions:} %
Similar to existing works \cite{mcmahan2017communication, li2019convergence, li2018federated}, we sample $K$ clients in each round $r$ (i.e., $K := \left| \mathcal{K}^{(r)} \right|$) where the sampling is uniform (without replacement) out of all $N$ clients.
We assume the communication and computation cost for a particular  device in each round %
is the same%
, but varies among devices due to system heterogeneity. We do not consider the cost for model aggregation in Line~\ref{alg:fedavgStep4}, because it only needs to compute the average that is much less complex than local model updates. %

\subsubsection{\textbf{Time Cost}}
For general heterogeneous systems, each client can have different communication and computation capabilities (see Fig.~\ref{fig:intro}). Let $t_k$ denote the per-round time for client $k$ to complete computation and communication. We have
\begin{equation}
\label{Troundi}
t_{k}=t_{k,p}  E+ t_{k,m} \ \forall k \in \{1, \ldots, N\},
\end{equation}
where $t_{k,p}$ %
is the com\underline{p}utation time for client $k$ to perform one local iteration, and $t_{k,m}$ %
is the per-round co\underline{m}munication time for a client to upload/download the model parameter.

Because the clients compute and communicate in parallel, for each round $r$, the per-round time $t^{(r)}$ depends on the slowest participating client (also known as straggler).\footnote{This is because in synchronized FL systems, the server needs to collect all updates from the sampled clients before performing global aggregation.}  Hence,%
\begin{equation}
\label{Tround}
t^{(r)}=\max _{k \in \mathcal{K}^{(r)}}\left\{t_{k}\right\}.%
\end{equation}
Therefore, the total learning time $t_\textnormal{tot}$ after $R$ rounds %
is
\begin{equation}
\label{Ttot}
t_\textnormal{tot}(K,E,R)=\sum\nolimits_{r=1}^{R}\max _{k \in \mathcal{K}^{(r)}}\left\{t_{k}\right\}.%
\end{equation}

\subsubsection{\textbf{Energy Cost}}
Similarly, by denoting $e_k$ as the per-round energy cost for client $k$ to complete the computation and communication, we have
\begin{equation}
\label{energyk}
e_{k}=e_{k,p}  E+ e_{k,m},
\end{equation}
where $e_{k,p}$ and $e_{k,m}$ are respectively the energy costs for client $k$ to perform a local iteration and a round of communication. 

Unlike the \emph{straggling} effect in time cost \eqref{Tround}, the energy cost $e^{(r)}$ in each round $r$ depends on the \emph{sum} energy consumption of the selected clients $\mathcal{K}^{(r)}$. %
Therefore, the total energy cost $e_\textnormal{tot}$ after $R$ rounds %
can be expressed as
\begin{equation}
\label{etot}
e_\textnormal{tot}(K,E,R)=\sum\nolimits_{r=1}^{R}\sum\nolimits_{k \in \mathcal{K}^{(r)}}e_{k}.%
\end{equation}

\subsection{Problem Formulation}
Considering the difference of the two cost metrics, the optimal solutions of $E$, $K$ and $R$ generally do not achieve the common goal for minimizing both $t_\textnormal{tot}$ and $e_\textnormal{tot}$. %
To strike the balance of learning time and energy consumption, we introduce a weight $\gamma \in \left[0, 1\right]$ and optimize the balanced  cost function in the following form: %
\begin{equation}
\label{homocost}
    C_\textnormal{tot}(K,E,R)= \left(1-\gamma\right)
    t_\textnormal{tot}(K,E,R)+ \gamma  e_\textnormal{tot}(K,E,R), 
    \end{equation}
{where $1-\gamma$  and $\gamma$ can be interpreted as the \emph{normalized price} of the two costs, i.e., how much monetary cost for one unit of time and one unit of energy, respectively. The value of $\gamma$ can be adjusted for different preferences.} %
For example, we can set $\gamma=0$ when all clients are plugged in and energy consumption is not a major concern, whereas $\gamma=1$ when devices are solar-based sensors where saving the devices' energy is the priority.   

Our goal is to minimize the {expected} total cost while ensuring convergence, %
which translates into this problem:
\begin{equation}\begin{array}{cl}
\label{ob1}
\!\!\!\!\!\!\!\!\textbf{P1:}\quad\quad \min_{E, K, R} & \Expect[C_\textnormal{tot}(E,K,R)] \\
\quad\quad \text { s.t. } & \Expect[F(\mathbf{w}_R)]-F^{*} \le \epsilon,\\
& K,E,R \in \mathbb{Z}^{+}, \  \text{and}\ \  1 \le K \le N. 
\end{array}\end{equation}
{where $\Expect[F(\mathbf{w}_R)]$ is the expected loss after $R$ rounds, %
$F^*$ is the (true and unknown) minimum value of $F$, and $\epsilon$ is the desired precision. We note that the expectation in \textbf{P1} is due to the randomness of SGD and client sampling in each round.}%

{Solving \textbf{P1} is challenging in two aspects. First, it is difficult to find an \emph{exact analytical expression} to relate  $E$, $K$ and $R$ with $C_\textnormal{tot}$, especially due to the \emph{non-linear maximum} function in $t_\textnormal{tot}$. %
Second, it is generally impossible to obtain an exact analytical relationship
to connect $E$, $K$ and $R$  with the convergence constraint. %
In the following section, we propose an algorithm that approximately solves \textbf{P1}, which we later show with extensive experiments that the proposed solution can achieve a  near-optimal performance of \textbf{P1}.}

\section{Cost-Effective Optimization Algorithm}
\label{sec:optimizationProblem}

This section shows how to approximately solve \textbf{P1}. 
{We first formulate an %
alternative problem that  
includes an approximate %
analytical relationship between the expected cost $\Expect[C_\textnormal{tot}]$, the convergence constraint, and the control variables $E$, $K$ and $R$. 
Then, we show that this new optimization problem %
can be efficiently solved after estimating %
unknown parameters associated with the convergence bound, and we propose a sampling-based algorithm to learn these unknown parameters.}%
\subsection{Approximate Solution to \textbf{P1}}
\subsubsection{\textbf{Analytical Expression of}  $\Expect[e_\textnormal{tot}]$} We first %
analytically establish the expected energy cost $\Expect[e_\textnormal{tot}]$ with $K$ and $E$.
\newtheorem{lemma}{{Lemma}}
\begin{lemma}
\label{lemma:expect_e_tot}
The expectation of $e_\textnormal{tot}$ in \eqref{etot} can be expressed as
\begin{equation}
\label{het_energy}
   \Expect[e_\textnormal{tot}(K,E,R)]=  K\left(e_{p}E+ e_{m}\right)R,
\end{equation}
where $e_p:=\frac{\sum_{k=1}^{N} e_{k,p}}{N}$ and $e_m:=\frac{\sum_{k=1}^{N} e_{k,m}}{N}$ denote the average per-device energy consumption for one local iteration and one round of communication, respectively.
\end{lemma}
\begin{proof}
Since all devices are sampled uniformly at random in each round, for $R$ rounds, each device will be sampled in $\frac{KR}{N}$ rounds in expectation. Given that each device $k$ consumes $e_{k,p}  E+ e_{k,m}$ energy in each round as shown in \eqref{energyk}, 
summing up $\frac{KR}{N}\left(e_{k,p}  E+ e_{k,m}\right)$ over all $N$ clients leads to \eqref{het_energy}.
\end{proof}

\subsubsection{\textbf{Analytical Expression of} $\Expect[t_\textnormal{tot}]$} %
Next, we show how to tackle the  straggling effect to %
establish the expected time cost $\Expect[t_\textnormal{tot}]$ with the control variables. %
{Without loss of generality, we reorder $\{t_k: \forall k\in\{1,2,...,N\}\}$, %
such that}
\begin{equation}
\label{reorder}t_{1} \leq t_{2} \leq \ldots \leq t_{k}  \leq \ldots \leq t_{N}.\end{equation}
\begin{lemma}
\label{lemma:expect_t_tot}
With the reordered $t_k$ %
as in \eqref{reorder}, the expectation of $t_\textnormal{tot}$ in \eqref{Ttot} can be expressed as\footnote{The notation of $C_{N}^{K}$ is also noted as  ${{N} \choose {K}}$ which represents the combination number of choosing $K$ out of $N$ without replacement.} \begin{equation}
 \begin{array}{c}
\label{avgtime}
\Expect[t_\textnormal{tot}(K,E,R)]=\frac{\sum\nolimits_{i=K}^N{C_{i-1}^{K-1} t_{i}} }{C_{N}^{K}}%
R.%
 \end{array}
\end{equation}
\end{lemma}
\begin{proof}
We omit the full proof due to page limitation. The idea is to show that the expectation of the per-round time in \eqref{Tround} is 
 \begin{equation}
  \begin{array}{c}
\label{average_heter}
\Expect[t^{(r)}] =\frac{1}{C_{N}^{K}} \sum\nolimits_{i=K}^N{C_{i-1}^{K-1} t_{i}}.%
 \end{array}
\end{equation}
We first use the recursive property of  $C_{m}^{n}\!+C_{m}^{n-1}\!=\!C_{m+1}^{n}$ to show that the number of total combinations for choosing $K$ out of $N$ devices ${C_{N}^{K}}$ can be extended as
${C_{N}^{K}}=\sum\nolimits_{i=K}^N{C_{i-1}^{K-1}}$. %
Then, each combination (e.g., $C_{N-1}^{K-1}$) corresponds to the number of a certain device (e.g., $N$) being the slowest one (e.g., $t_N$). %
Since all devices are sampled uniformly at random, taking the expectation of all combinations gives \eqref{average_heter}. %
\end{proof}

\subsubsection{\textbf{Analytical
Relationship Between $\Expect[C_\textnormal{tot}]$ and Convergence}}
Based on %
$\Expect[e_\textnormal{tot}]$ in \eqref{het_energy} and $\Expect[{t_\textnormal{tot}}]$ in \eqref{avgtime}, %
the objective function $\Expect[{C_\textnormal{tot}}]$ in \textbf{P1}
can be expressed as
\begin{equation}
\label{obj_fun} 
\begin{array}{c}
\Expect[{C_\textnormal{tot}}]\!= \!  \left(\frac{\left(1-\!\gamma\right)\sum\nolimits_{i=K}^N{C_{i-1}^{K-1} t_{i}} }{C_{N}^{K}}\!+\!\gamma K\left( e_pE+ e_m\right)\right)R.%
\end{array}
\end{equation}
To connect %
$\Expect[{C_\textnormal{tot}}]$ with the $\epsilon$-convergence constraint in \eqref{ob1}, we utilize the %
convergence result \cite{li2019convergence}: %
\begin{equation}
    \label{upperbound1}
    \begin{array}{c}
    \Expect[F(\mathbf{w}_R)]-F^{*}\le \frac{1}{ER}\left(A_0+B_0\left(1+
 \frac{N-K}{K(N-1)}\right) E^2\right),
 \end{array}
\end{equation} 
where $A_0$ and 
$B_0$ 
are loss function related constants characterizing the statistical  heterogeneity of non-i.i.d. data.
By letting the upper bound satisfy the convergence constraint,\footnote{We note that optimization using upper bound as an approximation has also been adopted in \cite{wang2019adaptive} and resource allocation based literature \cite{tran2019federated, chen2020convergence,yang2019energy}. Although the convergence bound is valid for strongly convex problems, our experiments demonstrate that the proposed method also works well for non-convex learning problems empirically.} and using \eqref{obj_fun} and Lemmas~\ref{lemma:expect_e_tot} and \ref{lemma:expect_t_tot}, we approximate \textbf{P1} as
\begin{equation}\begin{array}{cl}
\label{ob2}
\!\!\!\!\!\!\textbf{P2:} \  \min_{E, K, R} & \!\!\! \left(\frac{\left(1-\!\gamma\right)\sum\nolimits_{i=K}^N{C_{i-1}^{K-1} t_{i}} }{C_{N}^{K}}\!+\!\gamma K\left( e_pE+ e_m\right)\right) R \\
\quad \text { s.t. } &\!\!\! \frac{1}{ER}\left(A_0+B_0\left(1+
 \frac{N-K}{K(N-1)}\right) E^2\right) \le \epsilon\\
& K,E,R \in \mathbb{Z}^{+}, \  \text{and}\ \  1 \le K \le N.
\end{array}\end{equation}
Combining with \eqref{upperbound1}, we can see that \textbf{P2} is more constrained than \textbf{P1}, i.e., any feasible solution  of \textbf{P2} is also feasible for~\textbf{P1}. 

Problem \textbf{P2}, however, is still hard to optimize because it requires to compute various combinatorial numbers with respect to $K$. %
Moreover, %
even for a fixed value of $K$,
the combinatorial term %
is based on the reordering of $t_k$ in \eqref{Troundi}, which is uncertain as the order of $t_k$ changes with $E$. %
For analytical tractability, we further approximate \textbf{P2} as follows. %
\subsubsection{\textbf{Approximate Optimization Problem of \textbf{P2}}}
To address the complexity involved with computing the combinatorial term in \eqref{obj_fun}, similar to how we derive  \eqref{het_energy}, we define an approximation of $\Expect[t_\textnormal{tot}]$ as 
\begin{equation}
    \label{appro_t_tot}
    \tilde{\Expect}[t_\textnormal{tot}(E,R)] :=\left(t_pE+t_m\right)R,%
\end{equation}
where $t_p\!:=\!\frac{\sum_{k=1}^{N} t_{k,p}}{N}$ and $t_m\!:=\!\frac{\sum_{k=1}^{N} t_{k,m}}{N}$ are the average per-device time cost for one local iteration and one round of communication, respectively.
The approximation $\tilde{\Expect}[t_\textnormal{tot}]$ is equivalent to $\Expect[t_\textnormal{tot}]$ in the following two cases.

\emph{Case 1}: For homogeneous systems, where $t_{p}=t_{k,p}$ and $t_m= t_{k,m}, \forall k \in \{1, \ldots, N\}$,  we have %
\begin{equation}
\notag
\label{avgtimehomo}
\begin{array}{cl}
\Expect[{t_\textnormal{tot}}(K,E,R)]&=\left(t_pE+t_m\right)\frac{\sum\nolimits_{i=K}^N{C_{i-1}^{K-1} } }{C_{N}^{K}}%
R\\
&=\left(t_p E+t_m\right)R\\ &=\tilde{\Expect}[t_\textnormal{tot}(E,R)].
\end{array}
\end{equation}

\emph{Case 2}: For heterogeneous systems with
 $K\!=\!1$, we have %
\begin{equation}
\notag
\label{avgtimek=1}
\begin{array}{cl}
\Expect[{t_\textnormal{tot}}(K=1,E,R)]&=\left( \frac{t_{1}+ t_{2}\ldots+ t_{N}}{N}\right)R\\
&=\left( \frac{\sum_{k=1}^{N} t_{k,p}E+\sum_{k=1}^{N} t_{k,m}}{N}\right)R\\
&=\left(t_p E+t_m\right) R\\
&=\tilde{\Expect}[t_\textnormal{tot}(E,R)].
\end{array}
\end{equation}

Based on the approximation $\tilde{\Expect}[t_\textnormal{tot}(E,R)]$ in \eqref{appro_t_tot}, we formulate an approximate objective function of \textbf{P2} as
\begin{equation}
\label{objfun3}
\tilde{\Expect}[C_\textnormal{tot}(K,E,R)]=\left(1-\gamma \right)\tilde{\Expect}[t_\textnormal{tot}(E,R)]+\gamma \Expect[e_\textnormal{tot}(K,E,R)].
\end{equation}

Now, we relax $K$, $E$ and $R$ as continuous variables for theoretical analysis, which are rounded back to integer variables later. For the relaxed problem, if any feasible solution $E^\prime, K^\prime$, and $R^\prime$ satisfies the $\epsilon$-constraint in \textbf{P2} with inequality, we can always decrease this $R^\prime$ to some $R^{\prime\prime} < R^\prime$ which satisfies the constraint with equality but reduces the objective function value. Hence, for optimal $R$, the $\epsilon$-constraint is always satisfied with equality, and we can obtain $R$ from this equality as 
\begin{equation}
\begin{array}{c}
R = \frac{1}{\epsilon E}\left(A_0+B_0\left(1+
 \frac{N-K}{K(N-1)}\right) E^2\right).
\end{array}
 \label{eq:R_solution}
\end{equation}
By using $\tilde{\Expect}[C_\textnormal{tot}]$ to approximate $\Expect[C_\textnormal{tot}]$ and substituting \eqref{eq:R_solution} into its expression, we obtain %
\begin{equation}
\label{ob3}
\begin{array}{cl}
\!\!\!\!\!\!\!\!\!\!\!\!\!\!\!\textbf{P3:} \\
\min_{E, K}  & \!\!\!\!\!\frac{\left(\left(1-\gamma\right)\left(t_pE+t_m\right)+\gamma K \left(e_p E+e_m\right)\right) \cdot \left(\!A_0+B_0\left(\! 1+
 \frac{N-K}{K(N-1)}\!\right) E^2\right)}{\epsilon E}  \\[6pt]
\!\!\!\!\!\quad\text {s.t.} &  \!\!\! {E}\ge 1, \  \text{and}\ \  1 \le K \le N, %
\end{array}
\end{equation}
where we note that the objective function of \textbf{P3} is equal to $\tilde{\Expect}[C_\textnormal{tot}]$.
\textbf{P3} is an approximation of \textbf{P2} due to the use of $\tilde{\Expect}[C_\textnormal{tot}]$ to approximate the original objective $\Expect[C_\textnormal{tot}]$.

In the following, we solve \textbf{P3} as an approximation of the original \textbf{P1}. 
Our empirical results in Section~\ref{sec:experimentation} demonstrate that the solution obtained from solving \textbf{P3} %
achieves \emph{near-optimal performance} of the original problem \textbf{P1}. %
For ease of analysis, we incorporate $\epsilon$ in the constants $A_0$ and $B_0$ next.

\subsection{Solving the Approximate Optimization Problem  \textbf{P3}}
In this subsection, we first characterize some properties of the optimization problem \textbf{P3}. Then, we propose a sampling-based algorithm to learn the problem-related unknown parameters $A_0$ and $B_0$, based on which the  solution $K^*$ and $E^*$ (of \textbf{P3}) can be efficiently computed. The overall algorithm for obtaining $K^*$ and $E^*$ is given in Algorithm~\ref{alg:optimalSolution}.

\begin{algorithm}[t]
\small	\caption{Cost-effective design of $K$ and $E$}
\label{alg:optimalSolution}
	\KwIn{{$N$, $\gamma$, $t_p$, $t_m$, $e_p$, $e_m$, {loss $F_a$ and $F_b$,  $\mathbf{w}_0$, number of sampled pairs $M$}}, stopping condition $\epsilon_0$}
	\KwOut{{$K^*$ and $E^*$}}

\For{$i=1,2, \ldots, M$ \label{alg:optimalSolution:startEstimation}}{
Empirically choose $(K_i$, $E_i)$ and run Algorithm 1;

Record $R_{i,a}$ and $R_{i,b}$ when $F_a$ and $F_b$ are reached;
}

Calculate average $\frac{A_0}{B_0}$ using \eqref{A0B05}; \label{alg:optimalSolution:endEstimation}

Choose a feasible 
$z_0 \leftarrow \left(K_0, E_0\right)$ and set $j \leftarrow 0$; \label{alg:optimalSolution:startOptimization}

\While{$\Vert z_{j} - z_{j-1}\Vert > \epsilon_0$}{ %
Substitute $E_{j}$, $\frac{A_0}{B_0}$, $N$, $\gamma$, $t_p$, $t_m$, $e_p$, $e_m$ into  \eqref{opt_K} and derive $K^\prime$;

$K_{j+1} \leftarrow \arg\min_{K \in [1, N]} |K - K'|$; \label{alg:optimalSolution:projectionK}

Substitute $K_{j+1}$,  $\frac{A_0}{B_0}$, $N$, $\gamma$, $t_p$, $t_m$, $e_p$, $e_m$ into \eqref{partial_E} and derive $E^\prime$;

$E_{j+1} \leftarrow \arg\min_{E \geq 1} |E - E'|$; \label{alg:optimalSolution:projectionE}

$z_{j+1} \leftarrow \left(K_{j+1}, E_{j+1}\right)$ and $j \leftarrow j+1$;
}

Substitute four rounding combinations of $\left(\left \lceil{K_j}\right \rceil, \left \lceil{E_j}\right \rceil \right)$, $\left(\left \lceil{K_j}\right \rceil, \left \lfloor{E_j}\right \rfloor \right)$, $\left(\left \lfloor{K_j}\right \rfloor, \left \lceil{E_j}\right \rceil \right)$, and $\left(\left \lfloor{K_j}\right \rfloor, \left \lfloor{E_j}\right \rfloor \right)$ into the objective function of \textbf{P3}, and set the pair with the minimum value as $\left(K^*, E^*\right)$ \label{alg:optimalSolution:rounding}

\Return $\left(K^*, E^*\right)$%
\label{alg:optimalSolution:endOptimization}

\end{algorithm}

\subsubsection{\textbf{Characterizing \textbf{P3}}} 
The objective function of \textbf{P3} is non-convex because the determinant of its Hessian $\frac{\partial^2 \tilde{\Expect}[C_\textnormal{tot}]}{\partial^2 K} \frac{\partial^2 \tilde{\Expect}[C_\textnormal{tot}]}{\partial^2 E}-(\frac{\partial^2 \tilde{\Expect}[C_\textnormal{tot}]}{\partial K\partial E})^2$ is not always non-negative in the feasible set. However, the problem is \emph{biconvex}  \cite{gorski2007biconvex}.

\newtheorem{theorem}{{Theorem}}
\begin{theorem}
\label{theorem:biconvex}
Problem \textbf{P3} %
is strictly biconvex. %
\end{theorem}
\begin{proof}
For any $E \ge 1$, we have
{\small
 \begin{equation*}
\frac{\partial^2 \tilde{\Expect}[C_\textnormal{tot}]}{\partial^2 K}=\frac{2(1-\gamma)B_{0} N\left(t_{p} E^{2}+t_{m} E\right)}{(N-1) K^{3}}>0.\end{equation*} }
Similarly, for any $1\le K \le N$, we have
{\small
\begin{equation*}
\begin{array}{cl}
\dfrac{\partial^2 \tilde{\Expect}[C_\textnormal{tot}]}{\partial^2 E}\!\!\!\!\!\!&=2 \left( \left(1\!-\!\gamma\right)t_p+\gamma K e_p\right) B_{0}\left(1+\frac{N-K}{K(N-1)}\right)\\
&\ \ +
\dfrac{2A_{0}\left[ (1\!-\!\gamma)t_m+\gamma Ke_m\right] }{E^{3}}>0
\end{array}
\end{equation*}}%
Since the domain of $K$ and $E$ is convex as well,  we conclude that \textbf{P3} is strictly biconvex.
\end{proof}
The  biconvex property allows many efficient algorithms, such as \emph{Alternate Convex Search} (ACS) approach, to a achieve a guaranteed local optima\cite{gorski2007biconvex}. %
Nevertheless,  %
by analyzing the \emph{stationary point} of $\frac{\partial \tilde{\Expect}[C_\textnormal{tot}]}{\partial K}=0$
and $\frac{\partial \tilde{\Expect}[C_\textnormal{tot}]}{\partial E}\!=\!0$, we show that the optimal solution can be found %
more efficiently. 
This is because from $\frac{\partial \tilde{\Expect}[C_\textnormal{tot}]}{\partial K}\!=\!0$ %
we have $K$ in closed-form of $E$ as \begin{equation}
\label{opt_K}
\begin{array}{cl}
K= \sqrt{\frac{(1-\gamma)B_0N\left(t_{p}  E^3+ t_{m}E^2\right)}{\gamma\left[B_0(N-2)E^2+A_0(N-1)\right]
(e_{p}E+{ e_{m}})}}.
\end{array}
\end{equation}
By letting $\frac{\partial \tilde{\Expect}[C_\textnormal{tot}]}{\partial E}=0$, we derive the \emph{cubic equation} of $E$ as
\begin{equation}
\label{partial_E}
\begin{array}{cl}
\frac{2\left(1-\gamma\right)t_p+\gamma K e_p}{ 2\left(1-\gamma\right)t_{m}+\gamma Ke_m}  E^3
 +E^2 -{\frac{A_{0}}{B_0\left(1+\frac{N-K}{K(N-1)}\right)} }\!=\!0,
\end{array}
\end{equation}
which can be %
analytically solved in closed-form of $K$ via \emph{Cardano formula} \cite{schlote2005bl}.  
Therefore, for any fixed value of $K$,  due to biconvexity (Theorem~\ref{theorem:biconvex}), we have a unique real solution of $E$ from \eqref{partial_E} in closed form. Then, with ACS method we iteratively calculate \eqref{opt_K} and \eqref{partial_E} which keeps decreasing the objective function until we achieve the converged $K^\ast$ and $E^\ast$. This optimization process corresponds to Lines~\ref{alg:optimalSolution:startOptimization}--\ref{alg:optimalSolution:endOptimization} of Algorithm~\ref{alg:optimalSolution}, where Lines~\ref{alg:optimalSolution:projectionK} and \ref{alg:optimalSolution:projectionE} ensure that the solution is taken within the feasibility region, and Line~\ref{alg:optimalSolution:rounding} rounds the continuous values of $K$ and $E$ to integer values.

\subsubsection{\textbf{Estimation of Parameters $\frac{A_0}{B_0}$}}

Equations \eqref{opt_K} and \eqref{partial_E} include unknown parameters $A_0$ and $B_0$, which can only be determined during the learning process.\footnote{We assume that %
$t_p$, $t_m$, $t_m$ and $e_m$ can be measured offline.} In fact, $K$ in \eqref{opt_K} and $E$ in \eqref{partial_E} only depend on the value of $\frac{A_0}{B_0}$. %
In the following, we propose a sampling-based algorithm to estimate $\frac{A_0}{B_0}$, and show that the overhead for estimation is marginal. 

The basic idea is to sample different combinations of $\left(K, E\right)$ and use the upper bound in \eqref{upperbound1} to approximate $F(\mathbf{w}_R)\!-\!F^{*}$. 
{Specifically, we empirically sample\footnote{Our sampling criteria is to cover diverse combinations of  $\left(K,E\right)$.} a pair $\left(K_i, E_i\right)$ and run Algorithm 1 with an initial model $\mathbf{w}_0\!=\!\mathbf{0}$ until it reaches two pre-defined global losses $F_a := F(\mathbf{w}_{R_{i,a}})$ and $F_b := F(\mathbf{w}_{R_{i,b}})$ ($F_b<F_a$), where $R_{i,a}$ and ${R_{i,b}}$ are the executed round numbers for reaching losses $F_a$ and $F_b$. The pre-defined losses $F_a$ and $F_b$ can be set to a relatively high value, to keep a small estimation overhead, but they cannot be too high either as it would cause low estimation accuracy.} %
Then, we have%
\begin{equation}
\begin{cases}
    \label{A0B01}
    R_{i,a} \approx d + \frac{A_0 + B_0\left(1+
 \frac{N-K_i}{K_i(N-1)}\right) E_i^2}{E_i\left(F_a-F^{*}\right)},\\
    {R_{i,b}} \approx d + \frac{A_0 + B_0\left(1+
 \frac{N-K_i}{K_i(N-1)}\right) E_i^2}{E_i\left(F_b-F^{*}\right)}.
 \end{cases}
\end{equation}
from \eqref{upperbound1}, where $d$ captures a constant error of using the upper bound to approximate $F(\mathbf{w}_R)\!-\!F^{*}$. %
Based on \eqref{A0B01}, we have
 \begin{equation}
 \begin{array}{c}
    \label{A0B03}
    {R_{i,b}}-R_{i,a} \approx %
  \frac{\Delta}{E_i}  \left(    {A_0\! +\! B_0\left(1\!+\!
 \frac{N-K_i}{K_i(N\!-\!1)}\right)\! E_i^2}\right),
\end{array}
\end{equation}
where $\Delta :=\frac{1}{F_b\!-\!F^{*}}\!-\!\frac{1}{F_a\!-\!F^{*}}$. Similarly, sampling another pair of ($K_j$, $E_j$) and performing the above process gives us another executed round numbers $R_{j,a}$ and $R_{j,b}$. %
Thus, we have
 \begin{equation}
  \begin{array}{c}
    \label{A0B05}
    \frac{E_i\left(R_{i,b} -R_{i,a}\right)}{E_j\left(R_{j,b}-R_{j,a}\right)} \approx %
    \frac{ {A_0 + B_0\left(1+
 \frac{N-K_i}{K_i(N-1)}\right) E_i^2}}{{A_0 + B_0\left(1+
 \frac{N-K_j}{K_j(N-1)}\right) E_j^2}}.
  \end{array}
\end{equation}
We can obtain $\frac{A_0}{B_0}$ from \eqref{A0B05} (note that the variables except for $\frac{A_0}{B_0}$ are known).
In practice, %
we may sample several different pairs of $\left(K_i, E_i\right)$ %
to obtain an averaged estimation of $\frac{A_0}{B_0}$. This estimation process is given in Lines~\ref{alg:optimalSolution:startEstimation}--\ref{alg:optimalSolution:endEstimation} of Algorithm~\ref{alg:optimalSolution}.

\textbf{Estimation overhead}:
The main overhead for estimation comes from the additional iterations for the estimation of $\frac{A_0}{B_0}$.

For $M$ sampling pairs, the total number of iterations used for estimation is $\sum_{i=1}^MR_{i,b}E_i$, where $R_{i,b}$ is the number of rounds for sampling pair ($K_i$, $E_i$) to reach $F_b$. 
If the target loss is $F_R$ with the required number of rounds $R$, then according to \eqref{upperbound1}, {the overhead ratio can be written as
\begin{equation}
 \begin{array}{c}
    \label{A0B06}
    \frac{\sum_{i=1}^MR_{i,b}E_i}{RE^*} \approx  \frac{\sum_{i=1}^MR_{i,b}E_i\left({F_R-F^\ast}\right)}{A_0 + B_0\left(1+
 \frac{N-K^\ast}{K^\ast(N-1)}\right) \cdot ({E^\ast})^2},
  \end{array}
\end{equation}}%
where $K^*$ and $E^*$ are obtained from Algorithm 2. %
For a high precision with $F_R-F^{*}\rightarrow 0$, the overhead ratio is marginal.

\section{Solution Property for Cost Minimization}
\label{sec:property}

 We theoretically analyze the solution properties %
 for different metric preferences, which not only provide insightful design principles but also give alternative ways of solving \textbf{P3} to %
 more efficiently.  Our empirical results %
 show that these properties derived for \textbf{P3} \emph{are still valid for the original} \textbf{P1}. %
 In the following, we discuss the properties for $\gamma=0$ and $\gamma=1$, respectively. %
 For ease of presentation, we consider continuous $K$, $K^*$, $E$, and $E^*$ (i.e., before rounding) in this section. 
 
\subsection{Properties for Minimizing $\tilde{\Expect}[C_\textnormal{tot}]$ %
when $\gamma=0$}
When the design goal is to minimize learning time ($\gamma=0$),  %
the objective of \textbf{P3} can be rewritten as
\begin{equation}
 \begin{array}{c}
\label{ob4}
\min_{E, K} \  \left(\frac{t_m}{E}+t_p\right) \cdot \left(A_0+B_0\left(1+\frac{N-K}{K(N-1)}\right) E^2\right) 
 \end{array}
\end{equation}
We present the following insightful results to characterize the properties of $E^*$ and $K^*$. %
\begin{theorem}
When $\gamma=0$, $\tilde{\Expect}[C_\textnormal{tot}]$ is a strictly decreasing %
function in $K$ for any given $E$, hence $K^* = N$.
\end{theorem}
\begin{proof}
The proof is straightforward, as we are able to show for any $E$, $\frac{\partial\tilde{\Expect}[C_\textnormal{tot}]}{\partial K}<0$. %
Since $K \le N$, we have $K^*=N$.
The same result can also be obtained by letting $\gamma \rightarrow 0$ in \eqref{opt_K}. 
\end{proof}

\textbf{Remark:} In practical FL applications, $N$ can be very large, and thus, full participation ($K\!=\!N$) is usually intractable. %
However, 
since the objective function in \eqref{ob4} %
is strictly convex and decreasing with $K$, as $K$ increases, the marginal learning time decrease becomes smaller as well. %
Therefore, when $N$ is very large, 
sampling a small portion of devices %
can achieve a relative  good learning time. %
Our later real-data experiment %
shows that sampling $K\!=\!20$ out of $N\!=\!100$ devices achieves a similar performance as sampling all devices. %

Based on the above finding, it is important to analyze the property of $E$ when $K$ is chosen sub-optimally. In line with this, we present the following two corollaries.

\newtheorem{corollary}{Corollary}
\begin{corollary}
\label{corollary:t_tot_E1}
When $\gamma=0$, for any fixed value of $K$, as $E$ increases, $\tilde{\Expect}[C_\textnormal{tot}]$ first decreases and then increases.%
\end{corollary}
\begin{proof}
Taking the first order derivative of $\tilde{\Expect}[C_\textnormal{tot}]$ over $E$, 
\begin{equation}
 \begin{array}{c}
\label{partialE}
\frac{\partial\tilde{\Expect}[C_\textnormal{tot}]}{\partial E}=B_{0}\left( 2t_{p} E + t_{m} \right) 
\left(1+\frac{N-K}{K(N-1)}\right)-\frac{ t_{m} A_{0}}{E^{2}}.
 \end{array}
\end{equation}
Since $0 \leq \frac{N-K}{K(N-1)} \leq 1$ for any feasible $K$, \eqref{partialE} is negative when $E$ is small and positive when $E$ is large. 
\end{proof}
Corollary 1 shows that for any given $K$, $E$ should not be set too small nor too large for saving learning time.%

\begin{corollary}
\label{corollary:t_tot_E2}
When $\gamma=0$, for any fixed value of $K$, $E^*$ increases as $\frac{t_m}{t_p}$ increases. %
\end{corollary}
We omit the proof of Corollary~\ref{corollary:t_tot_E2} due to page limitation.
Intuitively, Corollary~\ref{corollary:t_tot_E2} says that for any given $K$, when $t_m$ increases or $t_p$ decreases,  %
the optimal strategy to reduce learning time is to perform more steps of iterations (i.e., increase $E$) before aggregation, which matches  the empirical observations for communication efficiency in \cite{mcmahan2017communication, yu2018parallel, stich2018local}.

\subsection{Properties for Minimizing $\tilde{\Expect}[C_\textnormal{tot}]$  %
when $\gamma=1$} 
When the design goal is to minimize energy consumption  %
($\gamma=1$), the objective of \textbf{P3} can be rewritten as
\begin{equation}
 \begin{array}{c}
\label{ob5}
\min_{E, K}  \left(\frac{e_mK}{E}\!+\!e_pK\right) %
\left(A_0\!+\!B_0\left(1\!+\frac{N-K}{K(N\!-\!1)}\right)\! E^2\right).
 \end{array}
\end{equation}
Besides the different metrics of $e_m$ and $e_p$, the key difference between \eqref{ob4} and \eqref{ob5} %
is the multiplication of $K$. Therefore, the main difference between $\gamma = 1$ and $\gamma=0$ is in the properties related to $K$, %
whereas the properties related to $E$ remain similar, which we show in the following.  %
\begin{theorem}
When $\gamma=1$, $\tilde{\Expect}[C_\textnormal{tot}]$ is a strictly increasing function in $K$ for any given $E$,  hence $K^* = 1$.
\end{theorem}
\begin{proof}
It is easy to show that $\frac{\partial \tilde{\Expect}[C_\textnormal{tot}]}{\partial K}>0$ for any given $E$.
Since $K\ge 1$, we have $K^*=1$. This conclusion can also be obtained when we let $\gamma = 1$ in \eqref{opt_K} since $1\leq K \leq N$. %
\end{proof}
\textbf{Remark:}  Theorem 3 shows that sampling fewer devices can reduce the total energy consumption, whereas according to Theorem~2, this results in a longer learning time. While this may seem contradictory at the first glance, we note that this result is correct because the total energy is the sum energy consumption of all selected clients. Although it takes longer time to reach the desired precision $\epsilon$ with a smaller $K$, there are also less number of clients participating in each round, so the total energy consumption can be smaller.

\begin{corollary}
When $\gamma\!=\!1$, for any fixed value of $K$, as $E$ increases, $\tilde{\Expect}[C_\textnormal{tot}]$ first decreases and then increases.%
\end{corollary}

\begin{corollary}
\label{corollary:e_tot_E1}
When $\gamma\!=\!1$, for any fixed value of $K$, $E^*$ increases as $\frac{e_m}{e_p}$ increases. %
\end{corollary}

The proofs and intuitions for Corollaries 3 and 4 are similar to Corollaries 1 and 2, which we omit  due to page limitation.

\subsection{Trade-off Between 
Learning Time and Energy Consumption}
In the above analysis, %
we derived a trade-off design principle for $K$, with a larger $K$ favoring learning time reduction,  while a smaller $K$  favoring  energy saving.  %
For a given $\gamma$, the optimal $K$ achieves the right balance between reducing learning time and energy consumption. %

\begin{theorem}
\label{K_E_gamma}
Assume that the power used for computation and communication are the same (i.e., $\frac{e_m}{t_m}=\frac{e_p}{t_p}$), then $K^*$ and $E^*$ both decrease as $\gamma$ increases. %
\end{theorem}

We omit the full proof due to page limitation. Intuitively, $\frac{e_m}{t_m}=\frac{e_p}{t_p}$ yields $\frac{t_p}{t_m}=\frac{e_p}{e_m}$. Hence, the quantities $\frac{t_p E + t_m}{e_p E + e_m}$ and $\frac{2\left(1-\gamma\right)t_p+\gamma K e_p}{ 2\left(1-\gamma\right)t_{m}+\gamma Ke_m}$ in \eqref{opt_K} and \eqref{partial_E}, respectively, remain unchanged regardless of the values of $K$, $E$, and $\gamma$. By some algebraic manipulations, we can see from \eqref{opt_K} and \eqref{partial_E} that when $\gamma$ increases or $E$ decreases, the value of $K$ according to \eqref{opt_K} decreases; when $K$ decreases, the solution of $E$ from \eqref{partial_E} also decreases (note that this solution remains unchanged regardless of $\gamma$). Therefore, whenever $\gamma$ increases, $K$ will decrease, then $E$ will decrease, and so on, until converging to a new $(K^*, E^*)$ that is smaller than before, and vice versa.

\begin{figure}[!t]
	\centering
	\includegraphics[width=8.cm,height=4cm]{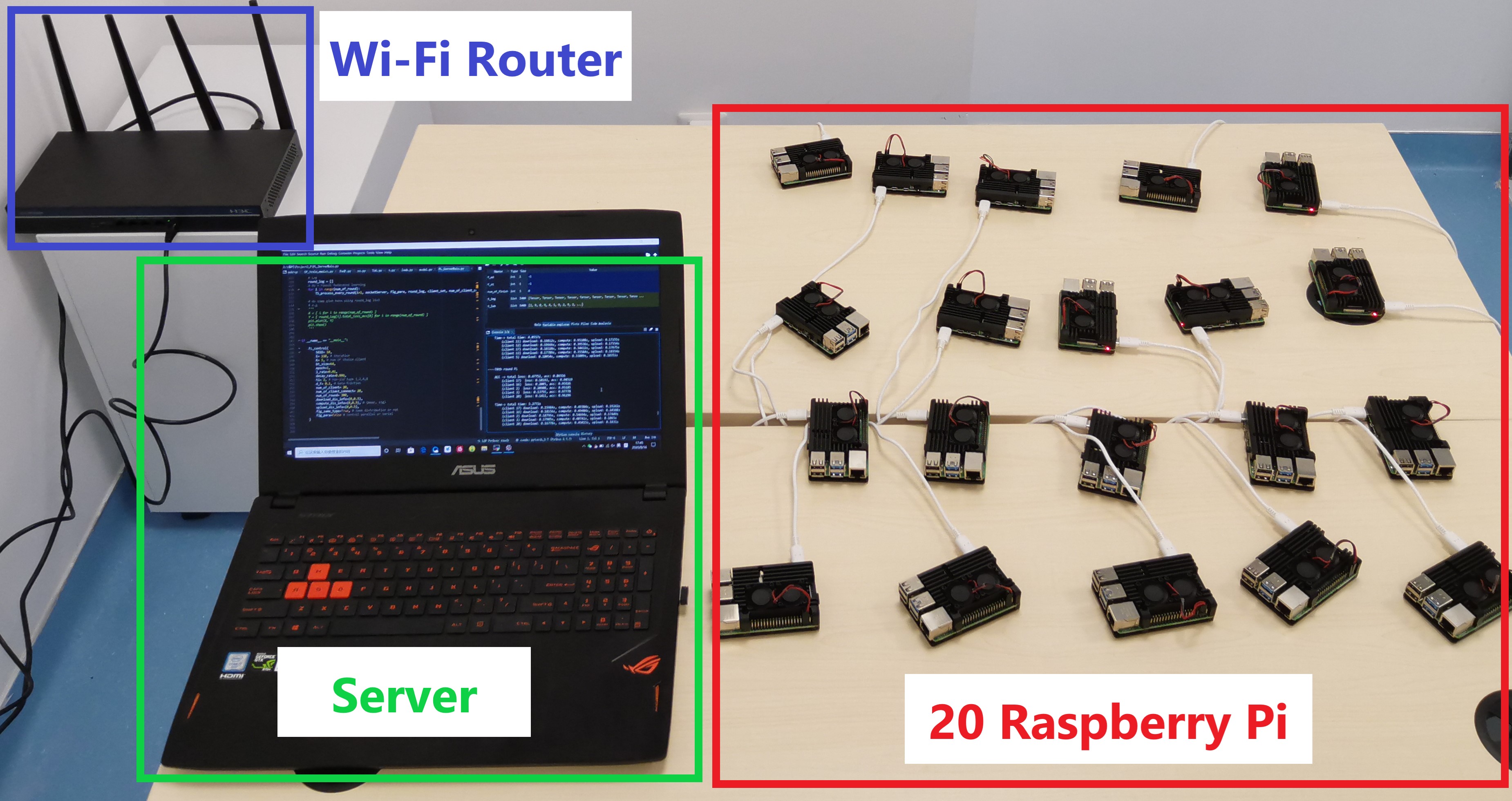}
	\caption{Hardware prototype with the laptop being central server, 20 Raspberry Pi being devices. During the FL experiments, the wireless router is placed 5 meters away from all the devices.}
\end{figure}

\section{Experimental Evaluation}
\label{sec:experimentation}

In this section, we evaluate the performance of our proposed cost-effective FL algorithm and verify our derived solution properties. We start by presenting the evaluation setup, and then show the experimental results.

\begin{table*}[!t]
\scriptsize
 \caption{Number of rounds for reaching estimation loss $F_a$ and $F_b$ for estimation of $\frac{A_0}{B_0}$ for three Setups \vspace{-0.05in}}
 \label{sample3}
  \centering
\begin{tabular}{c||c||c|c|c|c|c|c|c|c||c}
\toprule[1.2pt]
\multirow{3}{*}{\makecell[c]{Setup 1}}%
& \multirow{3}{*}{\makecell[c]{Estimation loss\\$F_a=0.6$\\\!$F_b=0.5$}}    & \scriptsize{Samples of $\left(K,E\right)$} & $\left(10, 50\right)$ &  $\left(15, 150\right)$ & $\left(20, 100\right)$ & $\left(10, 200\right)$ &$\left(20, 300\right)$  &-  & -& \multirow{3}{*}{{\makecell[c]{\textbf{Estimated}\\$\frac{A_0}{B_0}$=73,560}}  } \\ \cline{3-10}
                       &                       & \scriptsize{Rounds to achieve $F_a$} &  48 & 28 & 32 & 27 & 20 & -&- &                \\ \cline{3-10}
                       &                       & \scriptsize{Rounds to achieve $F_b$} &   82 & 46 & 56 & 43 & 35  & -& -&                 \\ \midrule[1.2pt]

\multirow{3}{*}{Setup 2}    & \multirow{3}{*}{\makecell[c]{Estimation loss\\$F_a=0.3$\\$F_b=0.2$}}    & Samples of $\left(K,E\right)$ &   $\left(10, 10\right)$ &  $\left(20, 20\right)$ & $\left(30, 30\right)$ & $\left(40, 40\right)$  & $\left(50, 50\right)$ &  $\left(60, 60\right)$ & $\left(80, 80\right)$ & \multirow{3}{*}{\makecell[c]{\textbf{Estimated}\\$\frac{A_0}{B_0}$=3,140}} \\ \cline{3-10}
                               &                       & Rounds to achieve $F_a$ &67  &37  &25  & 22 &  18&17  & 16 &                  \\ \cline{3-10}
                       &                       & Rounds to achieve $F_b$ &100 &60  & 39 &34  & 28 & 25   &24     \\ \midrule[1.2pt]
\multirow{3}{*}{Setup 3}    & \multirow{3}{*}{\makecell[c]{Estimation loss\\$F_a=1.5$\\$F_b=1.3$}}    & Samples of $\left(K,E\right)$&    $\left(10, 10\right)$ &  $\left(20, 20\right)$ & $\left(30, 30\right)$ & $\left(40, 40\right)$  & $\left(50, 50\right)$ &  $\left(60, 60\right)$ &$\left(80, 80\right)$ & \multirow{3}{*}{{\makecell[c]{\textbf{Estimated}\\$\frac{A_0}{B_0}$=3,750}}  } \\ \cline{3-10}
                           &                       & Rounds to achieve $F_a$  &  52& 39 &34  &31  &30  & 30 &29 &                  \\ \cline{3-10}
                       &                       & Rounds to achieve $F_b$  & 106 & 68 & 57 & 52 & 49 &48  &48 &                    \\ \bottomrule[1.2pt]
\end{tabular}
\end{table*}

\begin{figure*}[!t]
\centering
\subfigure[Loss with different %
$E$]{\label{hd_lossa}\includegraphics[width=3.53cm,height=3.1cm]{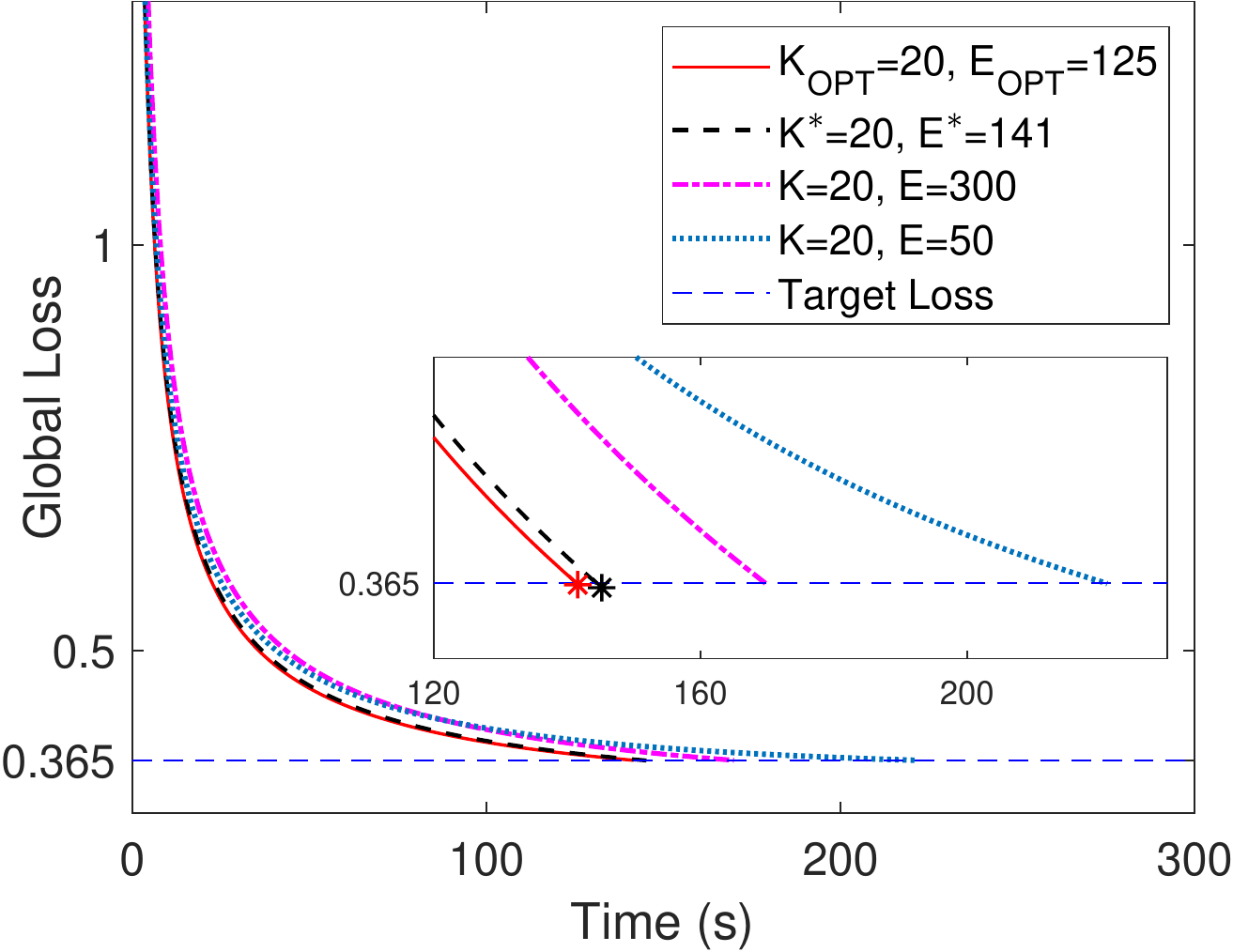}}
\subfigure[Loss with different %
$K$]{\label{hd_lossc}\includegraphics[width=3.53cm,height=3.1cm]{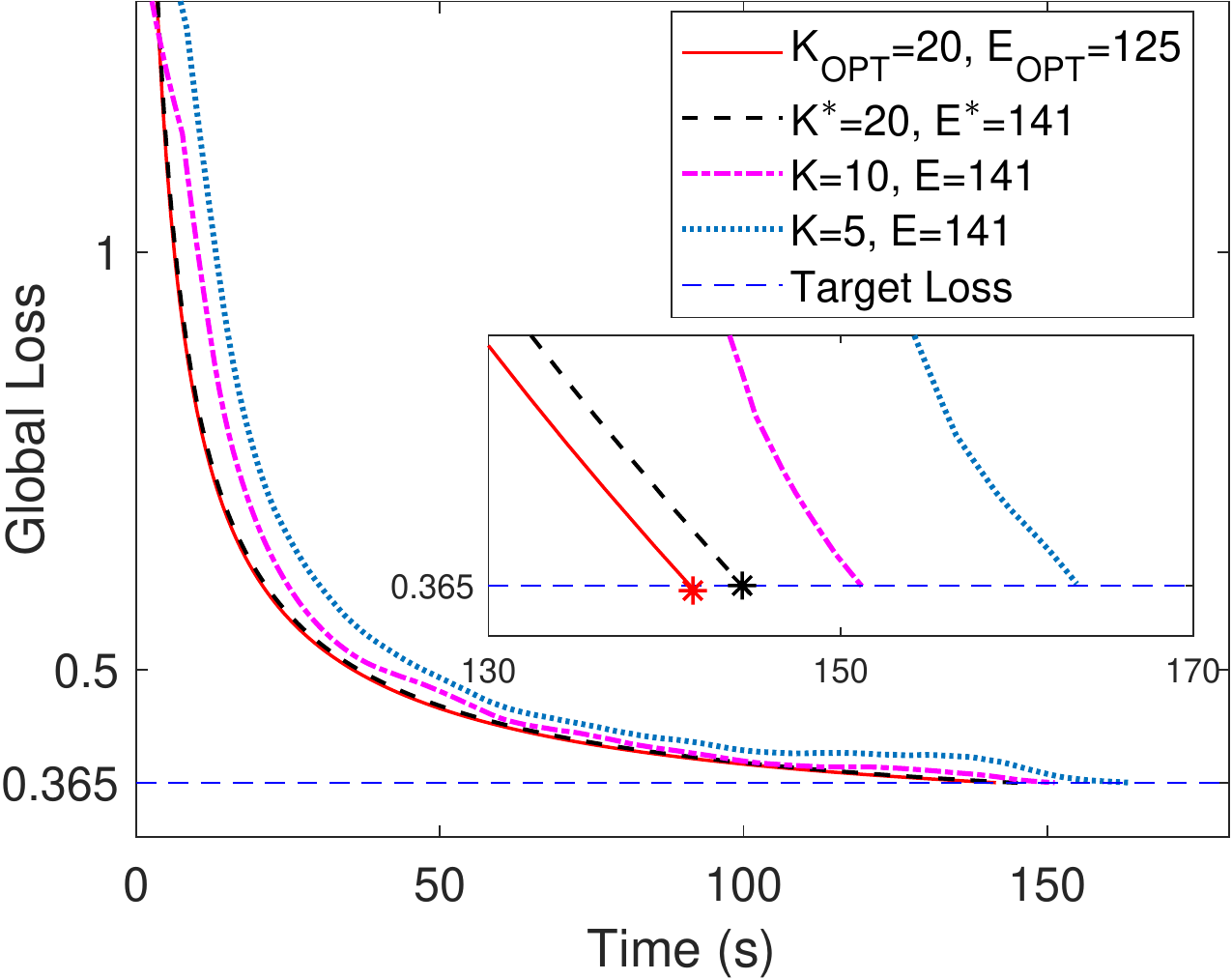}}
\subfigure[$t_\textnormal{tot}$ with different $\left(K, E\right)$ ]{\label{hd_t_tot}\includegraphics[width=3.57cm,height=3.1cm]{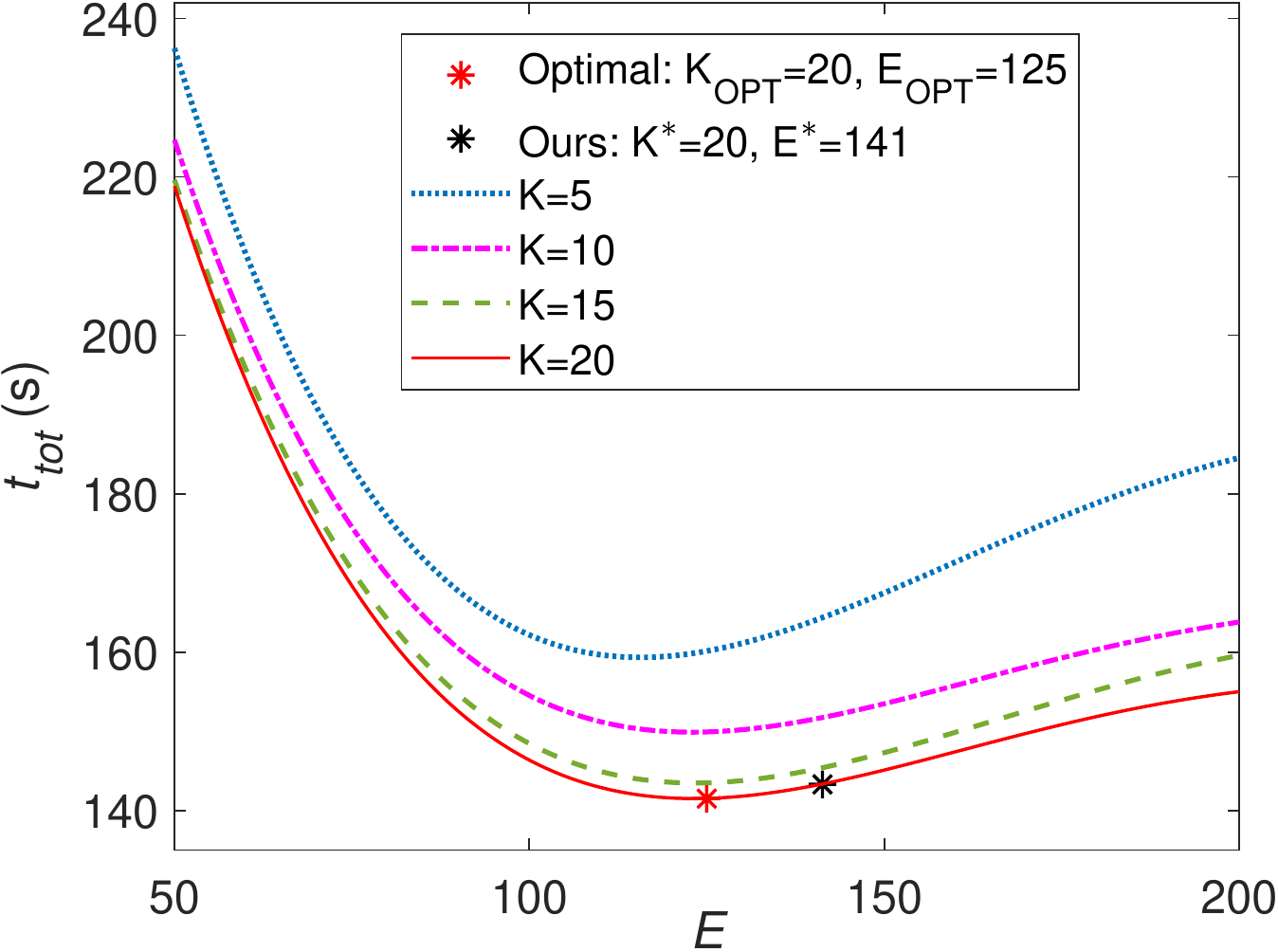}}
\subfigure[Accuracy with different $E$ ]{\label{hd_lossb}\includegraphics[width=3.53cm,height=3.1cm]{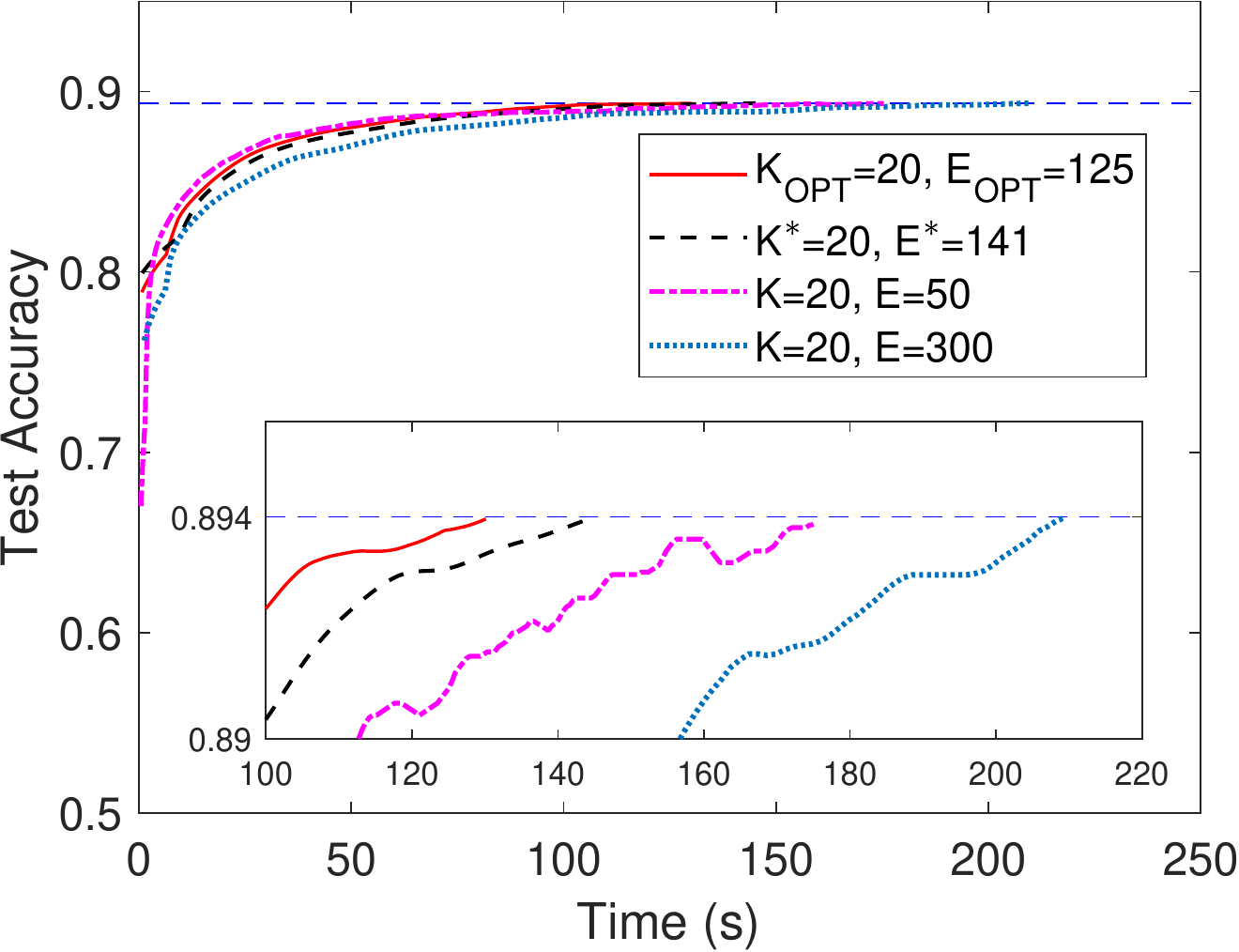}}
\subfigure[Accuracy with different $K$ ]{\label{hd_acc_t}\includegraphics[width=3.53cm,height=3.1cm]{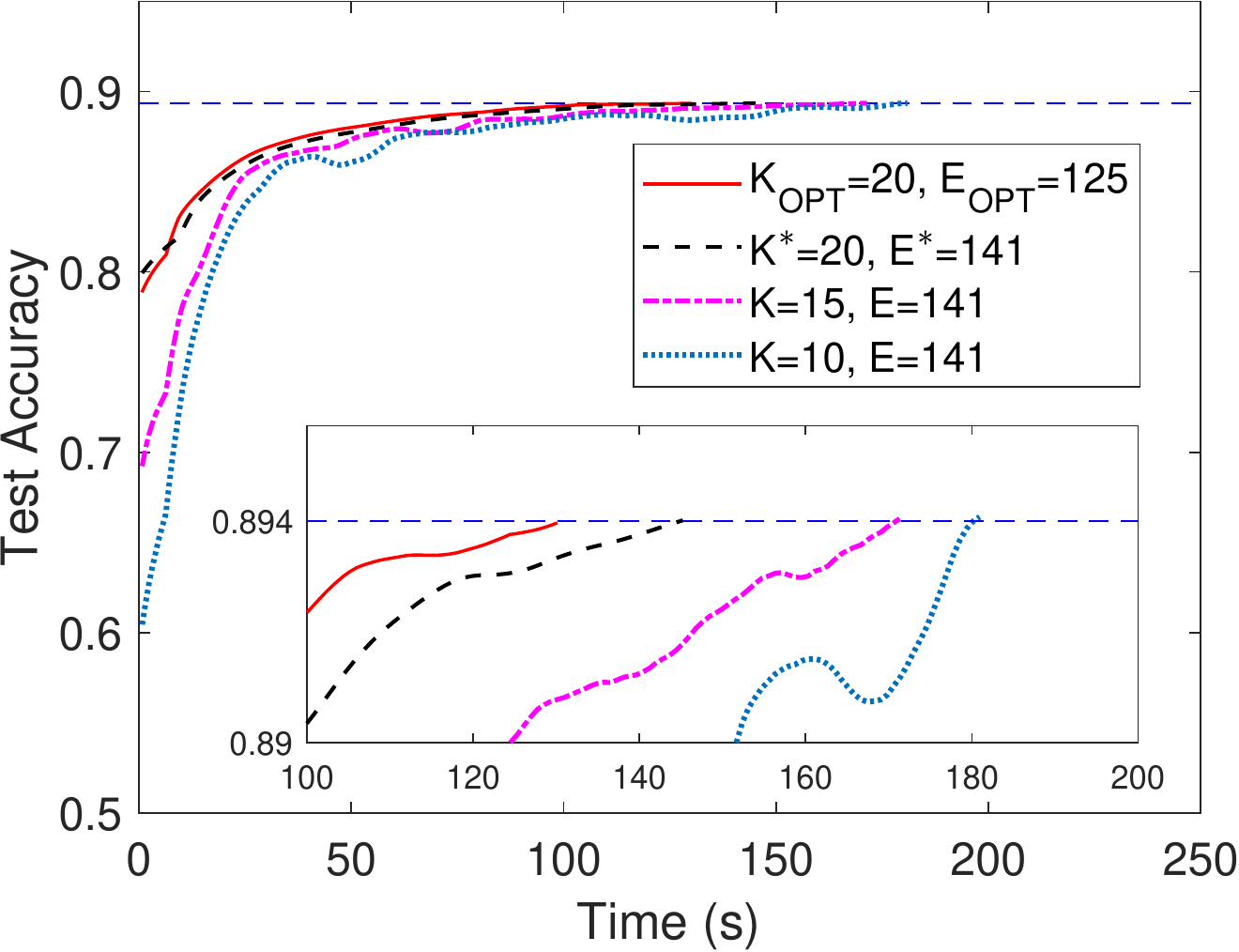}}
\caption{%
Training performance of \textbf{Setup 1} with logistic regression and MNIST for $\gamma\!=\!0$. (a)-(c): Our solution achieves the target loss $0.365$ using $145.2$s compared to the optimum $141.5$s with optimality error rate $2.61\%$, %
but faster than %
those with $E$ being too small or too large %
and those with $K$ being small. (d)-(e): 
 Our solution achieves $89.4\%$ test accuracy slightly longer than the optimal solution, but faster than %
 the non-optimal values of $\left(K,E\right)$ in (a) and (b).}
\label{hd}
\vspace{-0.1in}
\end{figure*}

\subsection{Experimental Setup}
\subsubsection{Platforms}
We conducted experiments both on a networked hardware prototype system and in a simulated environment. 
Our prototype system, as illustrated in Fig.~2,
consists of $N=20$ Raspberry Pis (version 4) serving as devices and a laptop computer serving as the central server. All devices are interconnected via an enterprise Wi-Fi router, and we developed a TCP-based socket interface for the peer-to-peer connection. In the simulation system, we simulated $N=100$ virtual devices and a virtual central server.  

\subsubsection{Datasets and Models} We evaluate our results both on a real dataset and a synthetic dataset. For the real dataset, we adopted the widely used MNIST dataset \cite{lecun1998gradient}, which contains square $28\times28=784$ pixel gray-scale images of $70,000$ handwritten digits ($60,000$ for training and $10,000$
for testing). For the synthetic dataset, we follow a similar setup to that in \cite{li2018federated}, which generates $60$-dimensional random vectors as input data. The synthetic data is denoted by $Synthetic \  (\alpha, \beta)$ with $\alpha$ and $\beta$ representing the statistical heterogeneity (i.e., how non-i.i.d. the data are). %
We adopt both the \emph{convex} {multinomial logistic regression} model \cite{li2019convergence}
and the \emph{non-convex} {deep convolutional neural network (CNN)} model with LeNet-5 architecture \cite{lecun1998gradient}.

\subsubsection{Implementation}
Based on the above, we consider the following three experimental setups.

\emph{Setup 1}: We conduct the first experiment on the prototype system using logistic regression and MNIST dataset, where we divide $6,000$ data samples (randomly sampled one-tenth of the total samples) among $N\!=\!20$ Raspberry Pis in a non-i.i.d. fashion with each device containing a balanced number of $300$ samples of only $2$ digits labels. %

    \emph{Setup 2}: We conduct the second experiment in the simulated system using CNN and MNIST dataset, where we divide all $60,000$ data samples among $N\!=\!100$ devices in the same non-i.i.d fashion as in Setup 1, but the amount of data in each device follows the inherent unbalanced digit label distribution of MNIST, where the number of samples in each device has a mean of $600$ and standard deviation of $20.1$. 
    
    \emph{Setup 3}: We conduct the third experiment in the simulated system using logistic regression  and $Synthetic \ (1, 1)$ dataset for statistical heterogeneity, 
    where we generate $24,517$ data samples and distribute them among $N\!=\!100$ devices in an unbalanced power law distribution,  where the number of samples in each device has a mean of $245$ and standard deviation of $362$.

\subsubsection{Training  Parameters}  For all experiments, we initialize our model with $\mathbf{w}_0=\mathbf{0}$ and SGD batch size $b=64$. In each round, we uniformly sample $K$ devices at random, which run $E$ steps of SGD in parallel.  For the prototype system, we use an initial learning rate $\eta_0=0.01$ with a fixed decay rate of $0.996$. For the simulation system, we use decay rate $\frac{{\eta}_0}{1+r}$, where $\eta_0=0.1$ and $r$ is communication round index.  We evaluate the aggregated model in each round on the global loss function. %
 Each result is averaged over 50 experiments.
\subsubsection{Heterogeneous System Parameters}
The prototype system allows us to capture real system heterogeneity in terms of communication and computation time, which we measured the average  $t_p\!=\!3.1 \times 10^{-3}$s with standard deviation $2.3 \times 10^{-4}$s and $t_m\!=\!0.34$s with standard deviation $1.56\times10^{-3}$s. %
We do not capture the energy cost in the prototype system because it is difficult to measure.
For the simulation system, we generate the learning time and energy consumption for each client $k$ using a normal distribution with
mean $t_p\!=\!0.1$s, $t_m\!=\!2$s, $e_p\!=\!10^{-3}$J, and $e_m\!=\!2 \times 10^{-2}$J and standard deviation of the mean divided by $3$. {According to the definition of $\gamma$, we  %
unify the time and energy costs such that one second is equivalent to $1-\gamma$ dollars~(\$) and one Joule is equivalent to $\gamma$ dollars~(\$).} %

\begin{figure*}[!ht]
\centering
\subfigure[\label{cnnlossK}Loss with different %
$E$]{\includegraphics[width=3.46cm,height=3.1cm]{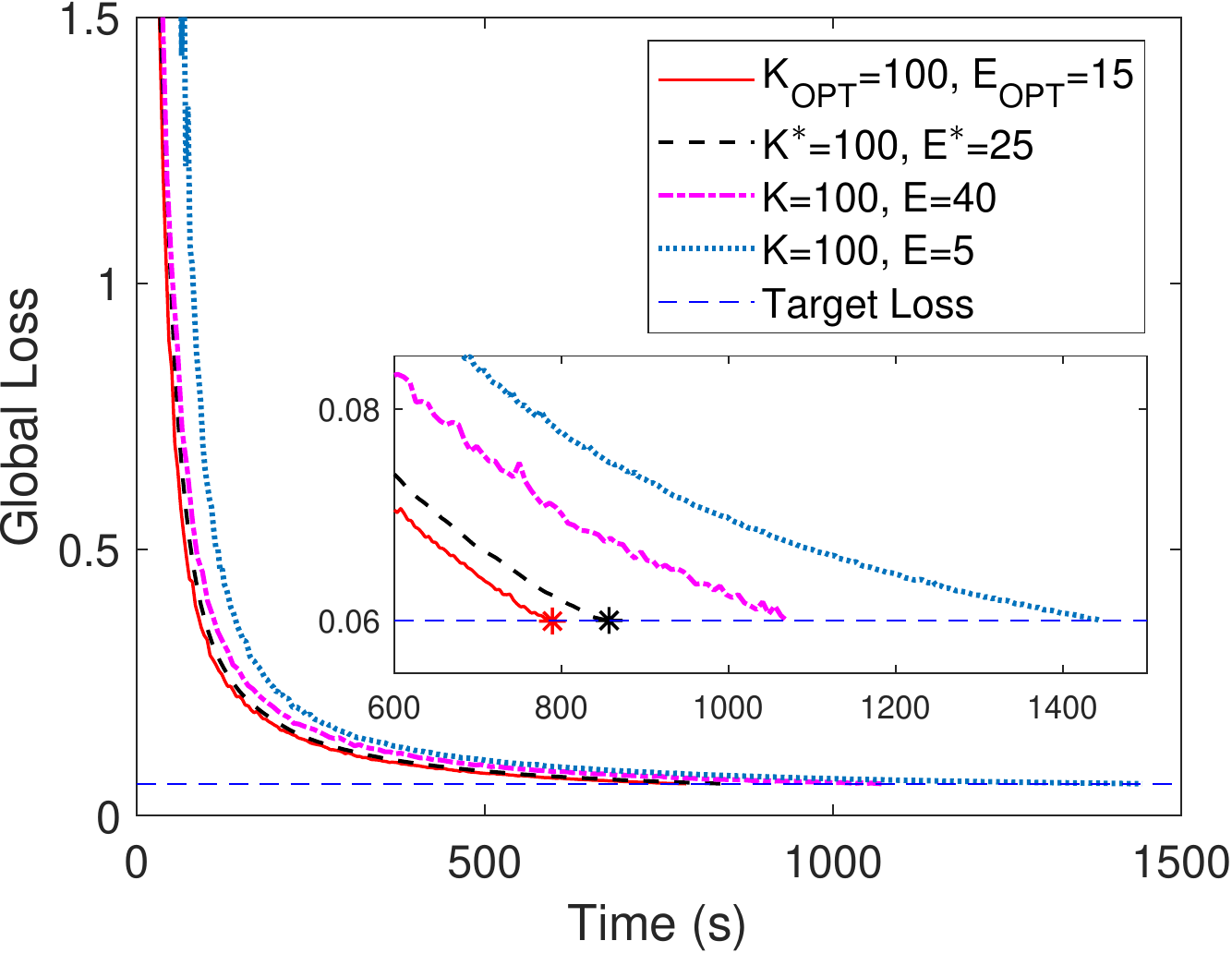}}
\subfigure[\label{cnnlossE}Loss with different $K$]{
\includegraphics[width=3.46cm,height=3.1cm]{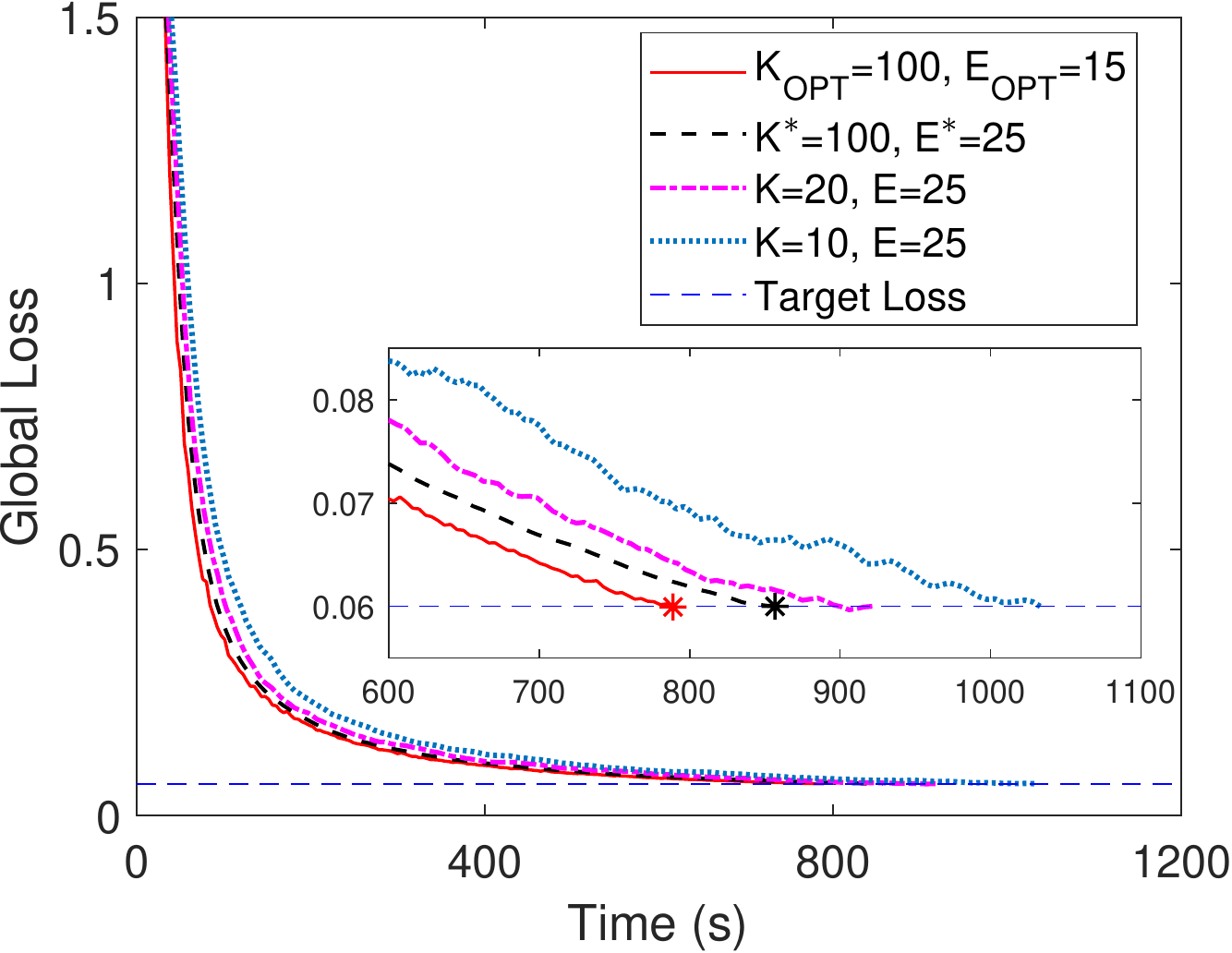}}
\subfigure[$t_\textnormal{tot}$ with different $\left(K, E\right)$ ]{\label{cnn_t_tot}\includegraphics[width=3.46cm,height=3.1cm]{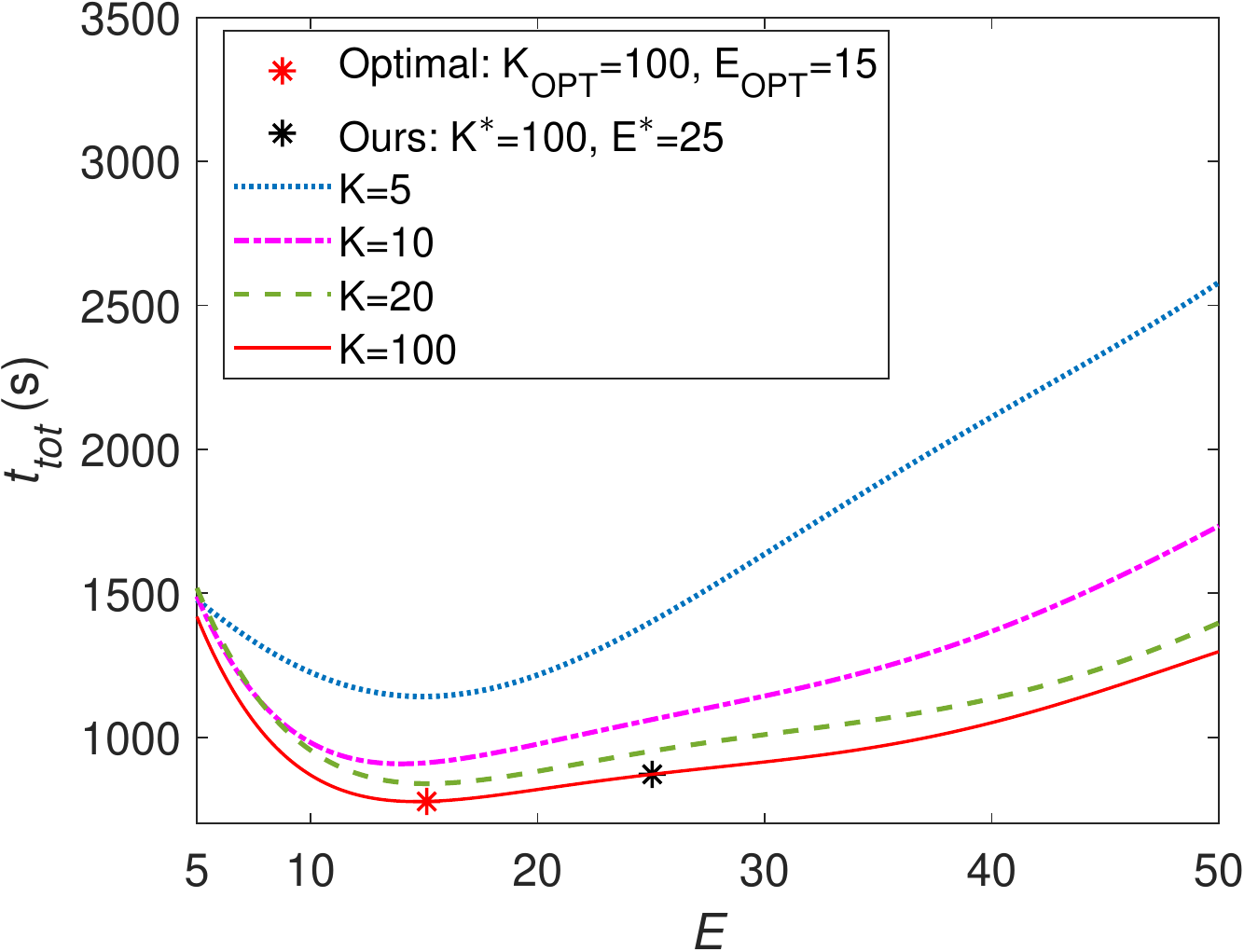}}
\subfigure[\label{cnnaccK}Accuracy with different $E$]{
\includegraphics[width=3.46cm,height=3.1cm]{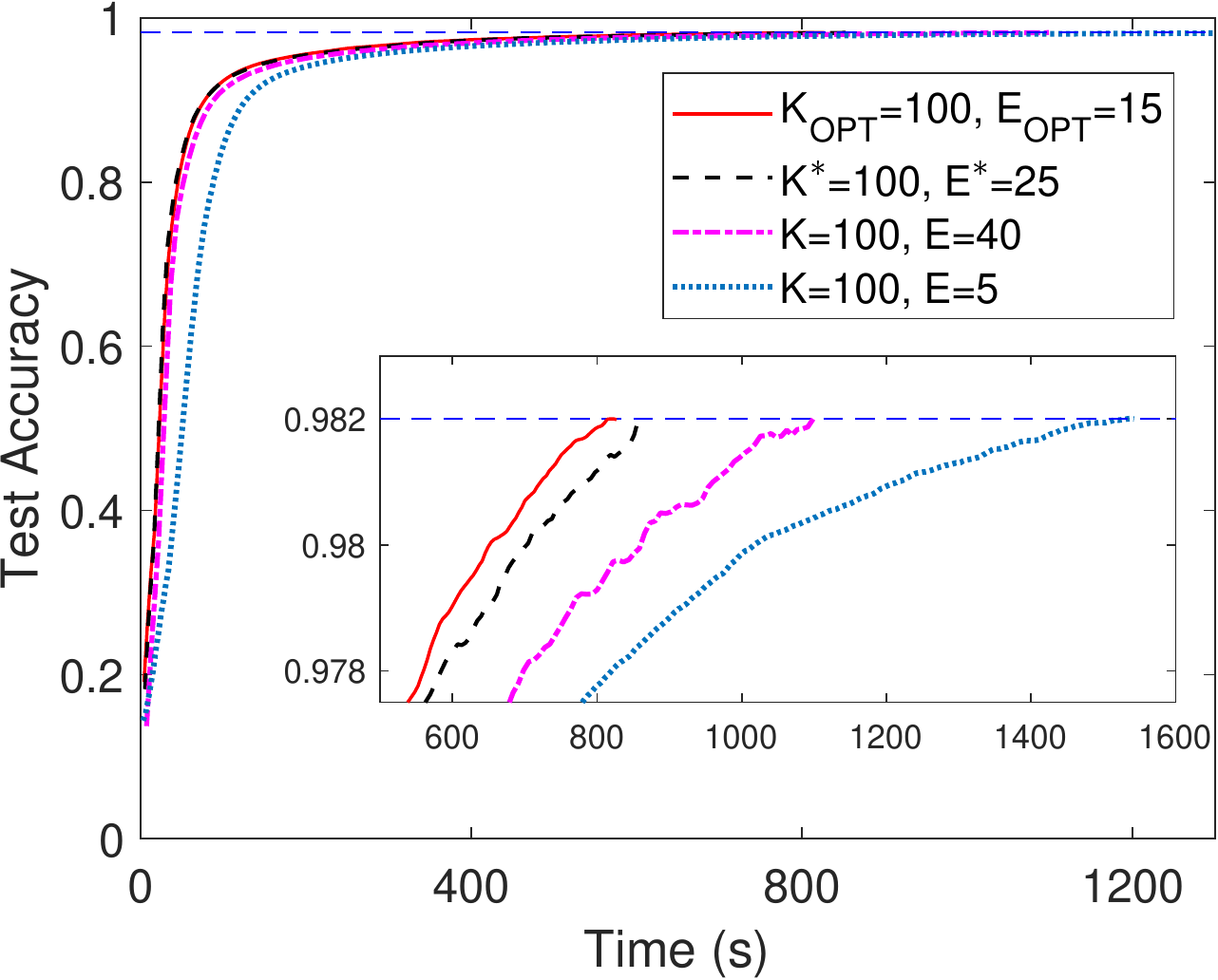}}
\subfigure[\label{cnnaccE}Accuracy with different $K$]{
\includegraphics[width=3.46cm,height=3.1cm]{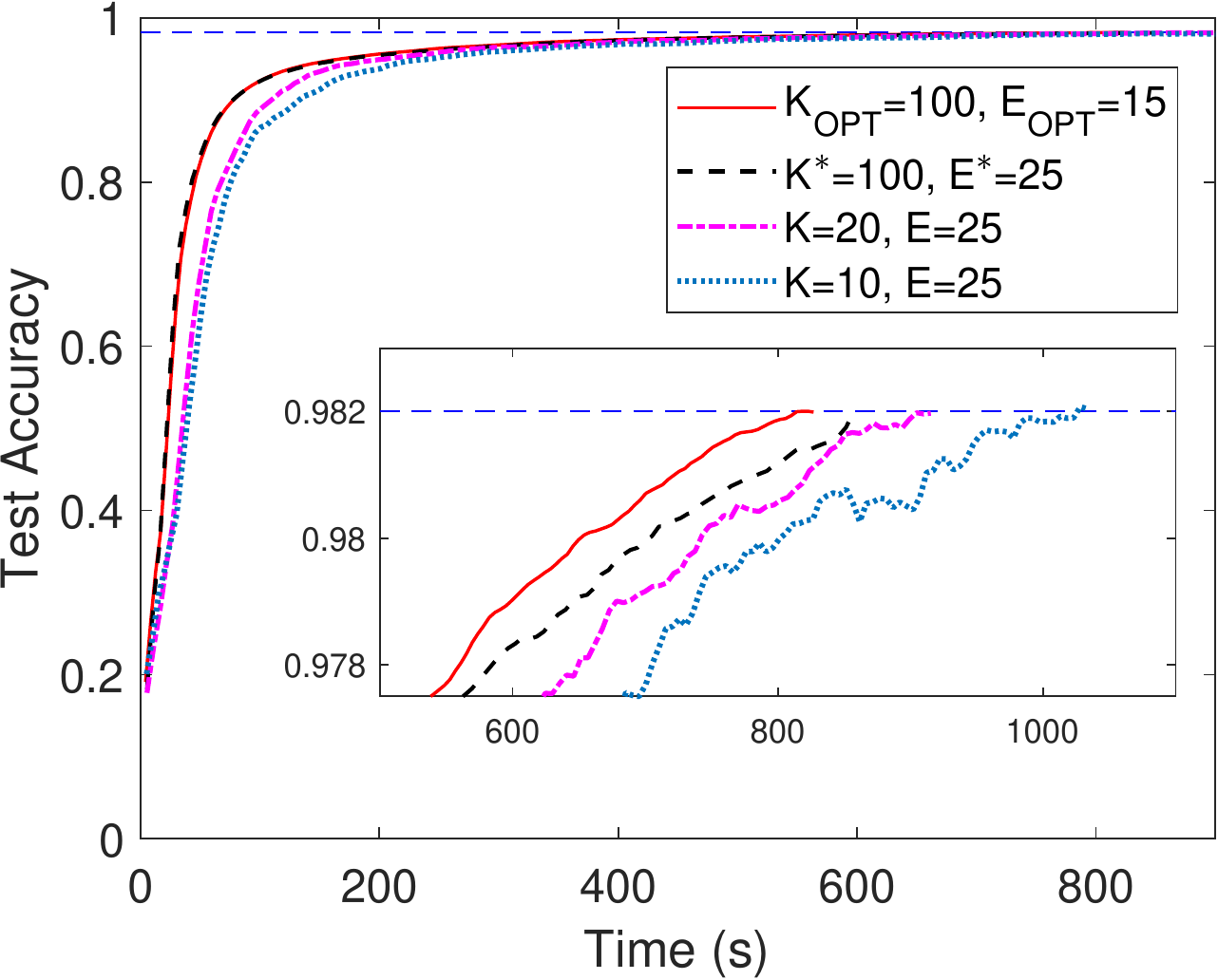}}
\caption{Training performance of \textbf{Setup 2} with CNN and MNIST for $\gamma\!=\!0$. (a)-(c): Our solution achieves the target loss $0.06$ using $856.8$s %
compared to the optimum $789.0$s 
with optimality error rate $8.49\%$, but faster than %
those with $E$ being too small or too large %
and those with $K$ being small. (d)-(e): 
 Our solution achieves $98.2\%$ test accuracy with almost the same time as the optimal solution, but faster than %
 the non-optimal values of $\left(K,E\right)$ in (a) and (b).}
\label{cnn}
\end{figure*}

\begin{figure*}[t]
\centering
\subfigure[\label{tot_K_E}$\gamma=0$]{
\includegraphics[width=3.53cm,height=3.1cm]{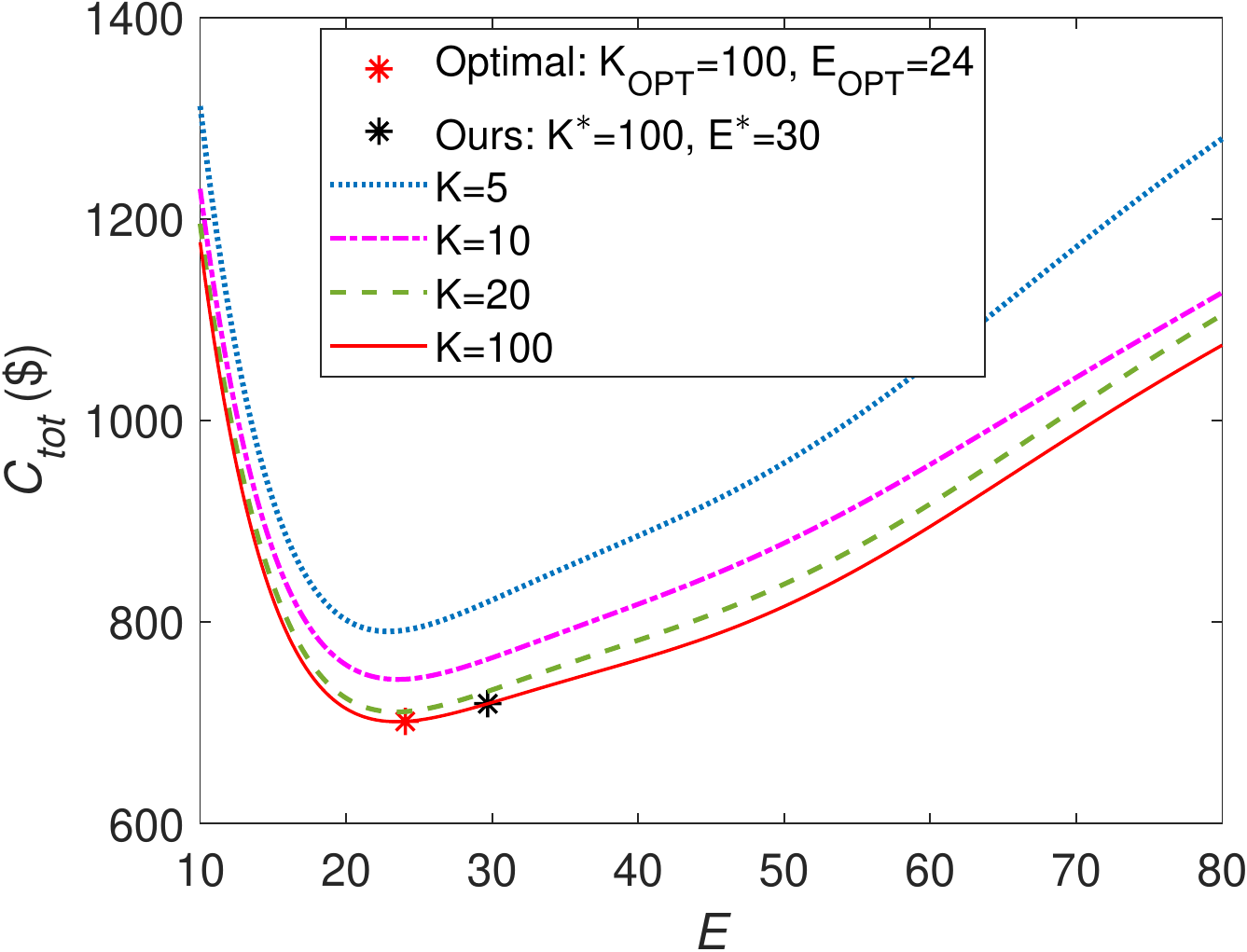}}
\subfigure[\label{ctot_K_E}$\gamma=0.45$]{\includegraphics[width=3.49cm,height=3.1cm]{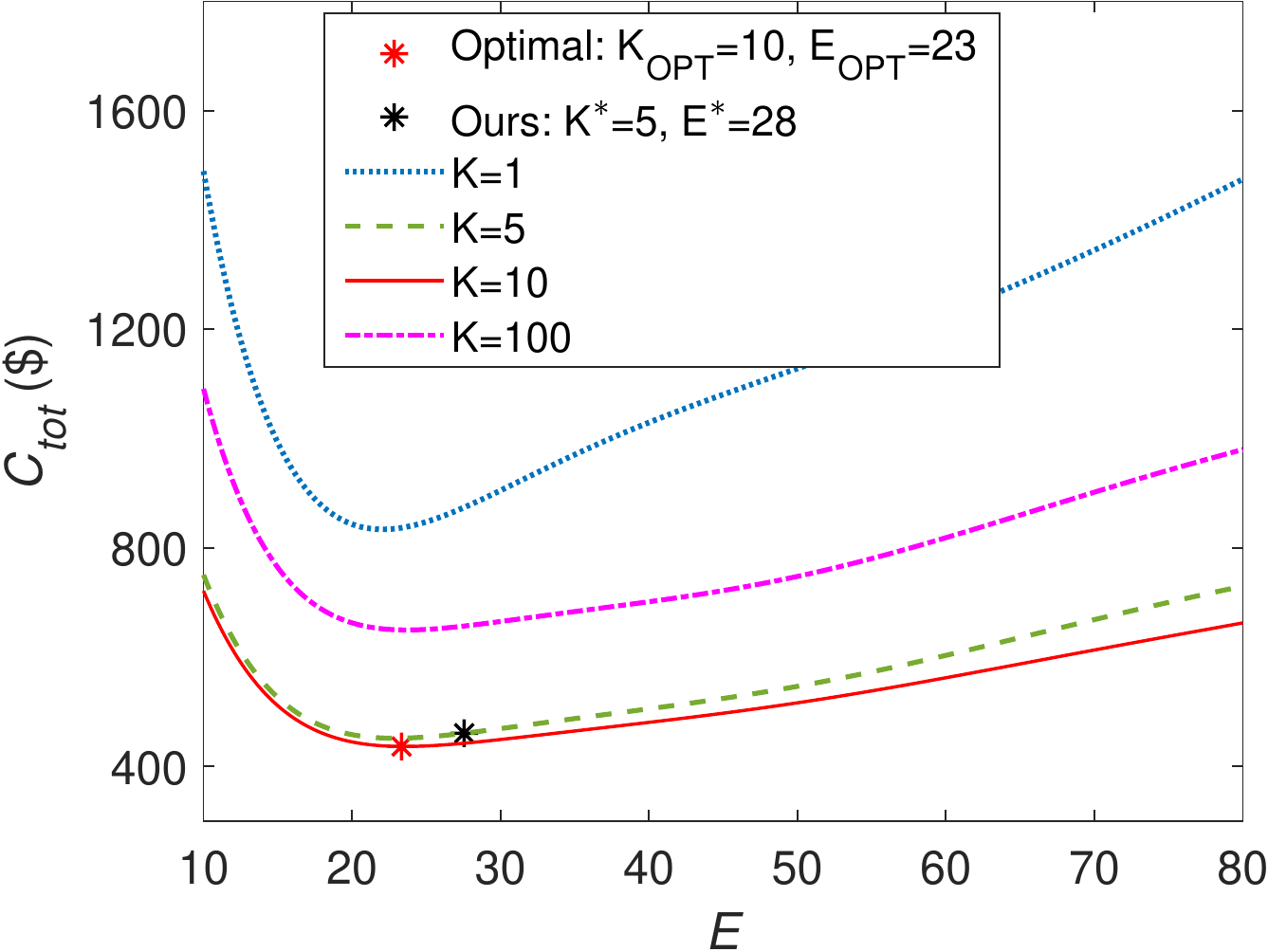}}
\subfigure[\label{etot_K_E}$\gamma=1$]{
\includegraphics[width=3.49cm,height=3.1cm]{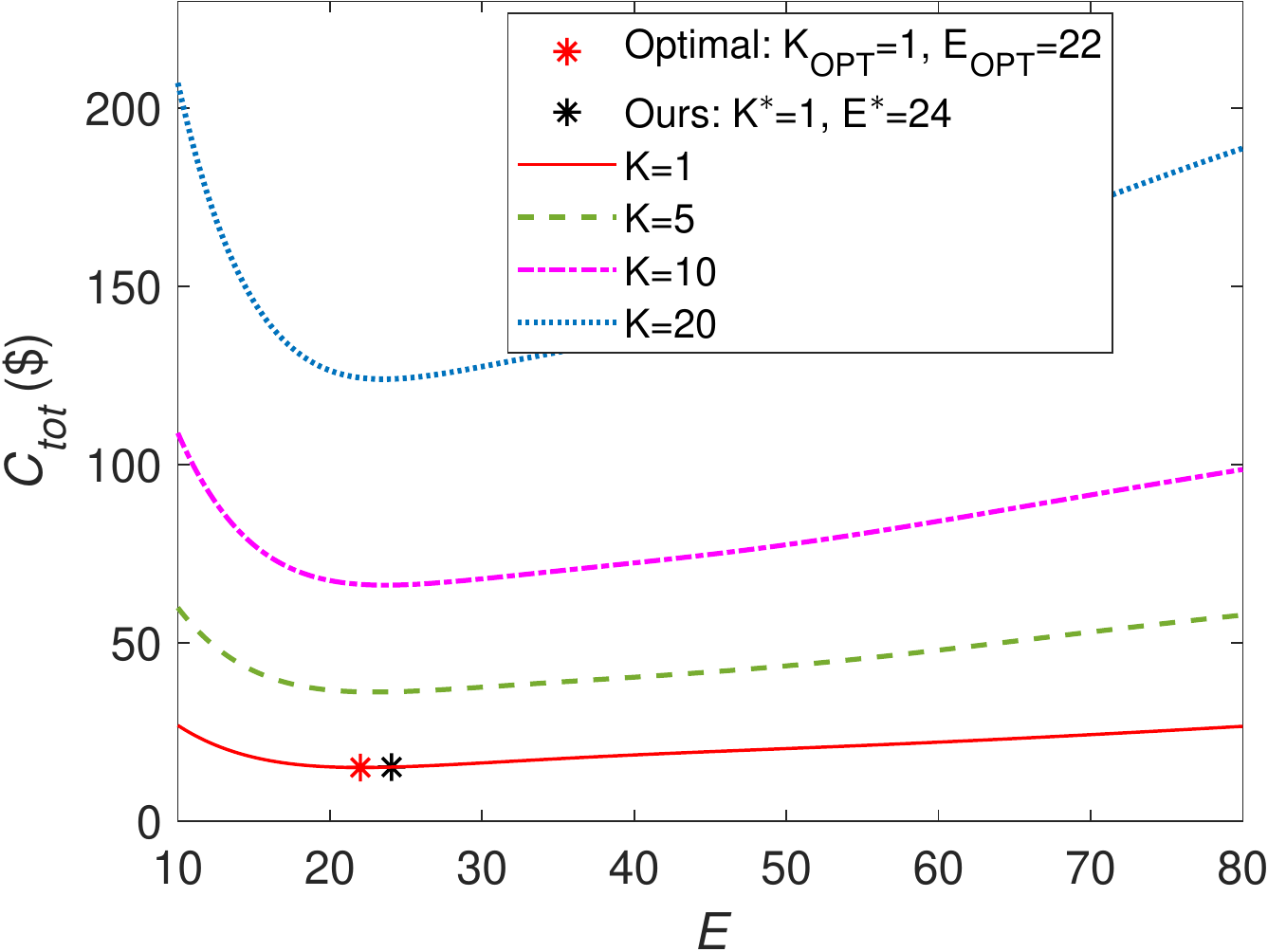}}
\subfigure[\label{tradeoffctot}$C_\textnormal{tot}$ with $\gamma$]{\includegraphics[width=3.47cm,height=3.1cm]{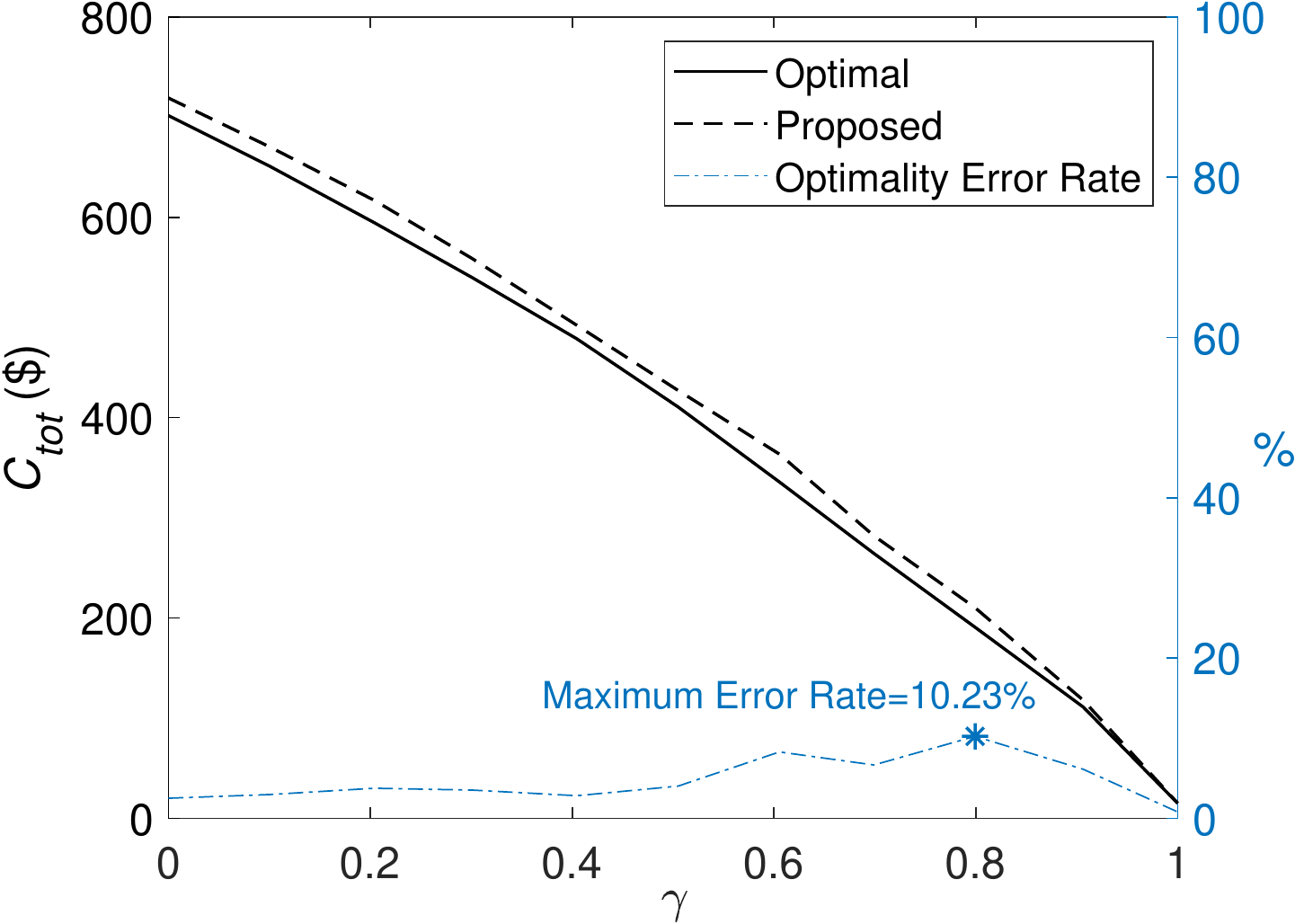}}
\subfigure[\label{tradeoff} $t_\textnormal{tot}$ and $e_\textnormal{tot}$ trade-off ]{\includegraphics[width=3.47cm,height=3.1cm]{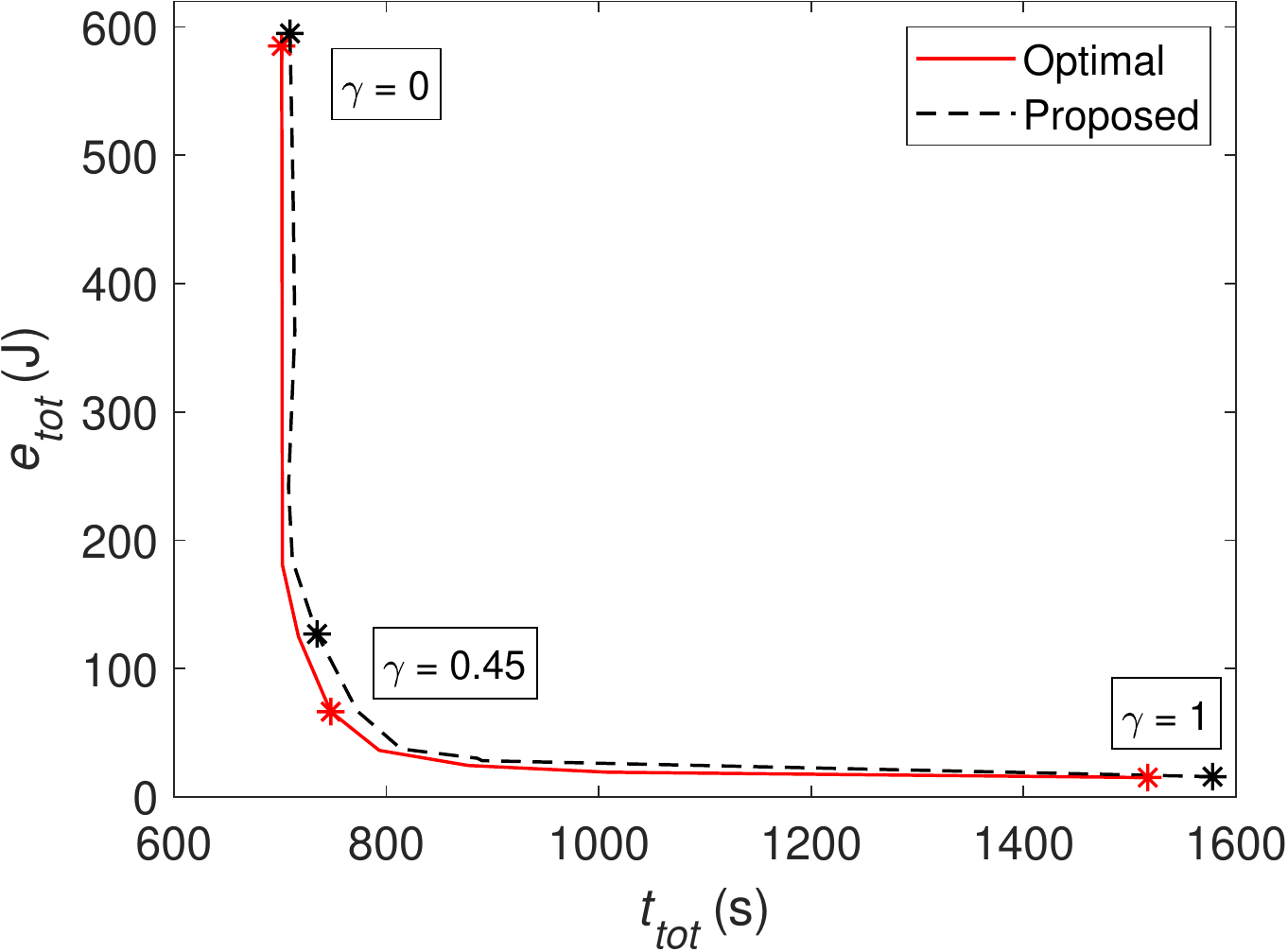}}
\caption{Performance of
$C_\textnormal{tot}$ for reaching the target loss 1.05  %
for %
\textbf{Setup 3} with logistic regression and Synthetic (1,1).  
(a)-(d) When %
$\gamma$ increases from $0$ to $1$, 
    our solutions reaches the target loss using the similar cost as the corresponding optimal ones, with  %
    average optimality error rate $4.85\%$ and maximum rate $10.23\%$.   
       (e) For different $\gamma$, our proposed solutions are able to balance the metric preferences between $t_\textnormal{tot}$ and $e_\textnormal{tot}$ while approaching the optimal trade-off.%
}
\label{c_tot}
\vspace{-0.05in}
\end{figure*}
\subsection{Performance Results}
We first validate the optimality of our proposed solution with estimated value of $\frac{A_0}{B_0}$. Then, we show the impact of the trade-off factor $\gamma$,  followed by %
the verification of our derived theoretical properties.

\subsubsection{Estimation of $\frac{A_0}{B_0}$}
We summarize the estimation process and results of $\frac{A_0}{B_0}$ for all three  experiment setups in Table~\ref{sample3}. 
Specifically, using Algorithm 2, we empirically\footnote{Due to the different learning rates, the sampling range of $E$ in Setup 1 is larger than those in Setup 2 and Setup 3.}  set  two relatively high target losses $F_a$ and $F_b$ with  a few sampling pairs of $\left(K,E\right)$. Then, we record the corresponding number of rounds for reaching $F_a$ and $F_b$, based on which  we calculate the averaged estimation value of $\frac{A_0}{B_0}$ using  \eqref{A0B05}. The proposed solution $K^*$ and $E^*$ is then obtained from Algorithm 2. %
For comparison, we denote %
$K_\textnormal{OPT}$ and $E_\textnormal{OPT}$ as the empirical optimal solution achieved by exhaustive search on %
 $\left(K, E\right)$. 

\subsubsection{Convergence and Optimality}
Figs.~\ref{hd} and \ref{cnn} show the learning time cost %
for reaching the target loss under different  $\left(K, E\right)$  for Setups 1 and 2, respectively.\footnote{For ease of presentation, we only show the convergence performance with $t_\textnormal{tot}$ 
for $\gamma \!= \!0$.}  
In both setups, {our proposed solutions %
achieve near-optimal performance} compared to the empirical optimal solution, with optimality error of $2.61\%$ and $8.49\%$, respectively. %
We highlight that our approach works well with the non-convex CNN model in Setup~2. %
Although the error rate in Setup 2 is higher than that in Setup~1, note that non-optimal values of $(K,E)$ without optimization %
may increase the learning time by several folds. 

Fig.~\ref{c_tot} depicts the  performance of $C_\textnormal{tot}$ for achieving the target loss with $\gamma$  under different $\left(K, E\right)$ in Setup 3, where our proposed solutions achieve near-optimal performance %
for all values of $\gamma$. %
Particularly, Fig.~\ref{tradeoffctot} shows that, throughout the range of $\gamma$, our approach has 
the maximum optimality error of $10.23\%$ and an average error of $4.85\%$.

\subsubsection{Impact of Weight $\gamma$}
Figs.~\ref{tot_K_E} -- \ref{etot_K_E} show that when the weight $\gamma$ increases from $0$ to $1$ (corresponding to the design preference varying from reducing learning time to saving energy consumption), both the optimal and our proposed solutions of $K$ decrease from $N\!=\!100$ to $1$. At the same time, both the optimal and our proposed solutions of $E$, although with a small difference, decrease slightly as $\gamma$ increases. We give the explanations of these observations in Section~\ref{sec:propertyValidation} below. %
 By iterating through the entire range of $\gamma$,  Fig.~\ref{tradeoff} depicts the trade-off curve between the learning time cost and energy consumption cost, where our algorithm is capable of balancing the two metrics as well as approaching the optimal.

\subsubsection{Property Validation} 
\label{sec:propertyValidation}
We highlight that our derived theoretical properties of $K$ and $E$ in Section~\ref{sec:property} can be validated empirically, %
which we summarize as follows.\footnote{The property of Corollaries~\ref{corollary:t_tot_E2} and \ref{corollary:e_tot_E1}  can also be validated, which we do not show in this paper due to page limitation.}

\begin{itemize}
  \item   Figs.~\ref{hd_t_tot}, \ref{cnn_t_tot}, and \ref{tot_K_E} demonstrate that %
  for any fixed value $E$, %
    the learning time cost ($\gamma\!=\!0$) %
    strictly decreases in $K$,  which confirms the claim in Theorem 2 that  sampling  more  clients  can speed up learning. 
   Moreover, we observe in these figures that sampling  fewer clients (15 out of 20 for Setup 1 and 20 out of 100 for Setups 2 and 3) does not affect the learning time much, which  confirms our  Remark. %
  Nevertheless, Fig.~\ref{etot_K_E} shows that %
  for any fixed value $E$,  the energy consumption cost ($\gamma\!=\!1$) %
  strictly increases in $K$, which confirms Theorem 3 that sampling fewer clients reduces energy consumption. 
  
  \item  Figs.~\ref{hd_t_tot}, \ref{cnn_t_tot}, \ref{tot_K_E}--\ref{etot_K_E} demonstrate that,  for any fixed value of $K$, %
   the corresponding cost first decreases and then increases as $E$ increases, which confirms Corollaries~1 and 3 as well as the biconvex property in Theorem~1. {Since $\frac{e_m}{t_m}=\frac{e_p}{t_p}=10^{-2}$ in our simulation system, we observe from Figs.~\ref{hd_t_tot}, \ref{cnn_t_tot}, and \ref{tot_K_E} that both $K^\ast$ and $E^\ast$  decrease as $\gamma$ increases, which confirms Theorem~\ref{K_E_gamma}.}

\end{itemize}

\section{Conclusion}
\label{sec:conclusion}

In this work, we have studied the cost-effective design for FL.
We analyzed how  to optimally choose the  number  of participating  clients  ($K$)  and  the  number  of  local  iterations ($E$), which are two essential control variables in FL, to minimize the total cost  %
while ensuring convergence. 
We proposed a sampling-based control algorithm which efficiently solves the optimization problem with marginal overhead. We also derived 
insightful solution properties which helps identify the design principles for different optimization goals, e.g., reducing learning time or saving energy. 
Extensive experimentation results validated our theoretical analysis and demonstrated the effectiveness and efficiency of our control algorithm. 
Our optimization design is orthogonal to most works on resource allocation for FL systems, and can be used together with those techniques to further reduce the cost.


\bibliographystyle{IEEEtran}
\bibliography{ref}
\end{document}